\documentclass[10pt]{article}

\PassOptionsToPackage{numbers, compress}{natbib}

\usepackage{natbib}
\usepackage[utf8]{inputenc} 
\usepackage[T1]{fontenc}    
\usepackage{hyperref}       
\usepackage{url}            
\usepackage{booktabs}       
\usepackage{amsfonts}       
\usepackage{nicefrac}       
\usepackage{microtype}      
\usepackage{xcolor}         
\usepackage{amsthm}
\usepackage{amsmath}
\usepackage{pifont}
\usepackage{array}
\usepackage{amssymb}
\usepackage{xcolor}
\usepackage{graphicx}
\usepackage[margin=1in]{geometry}
\usepackage[toc,page,header]{appendix}
\usepackage{minitoc}

\newcommand{\cmark}{\ding{51}}%
\newcommand{\xmark}{\ding{55}}

\newtheorem{lemma}{Lemma}
\newtheorem{corollary}{Corollary}
\newtheorem{theorem}{Theorem}

\newtheorem{proposition}{Proposition}
\theoremstyle{definition}
\newtheorem{definition}{Definition}
\theoremstyle{remark}
\newtheorem{remark}{Remark}
\newtheorem{fact}{Fact}

\DeclareMathOperator{\tr}{Tr}
\DeclareMathOperator{\Div}{Div}

\DeclareMathOperator{\Det}{Det}

\newcommand{\cube}{C_d}
\renewcommand{\d}{\mathrm{d}}

\newcommand{\R}{{\mathbb{R}}}

\newcommand{\Onot}{\mathcal{O}}
\renewcommand{\P}{\mathbb{P}}
\newcommand{\Q}{\mathbb{Q}}

\newcommand{\mc}{\mathcal}

\newcommand{\p}{\varphi}
\newcommand{\lin}{\mathrm{lin}}
\newcommand{\conf}{\mathrm{conf}}
\newcommand{\sconf}{\mathrm{r-conf}}
\newcommand{\ima}{\mathrm{OCT}}
\newcommand{\rep}{\mathrm{reparam}}
\newcommand{\vertiii}[1]{{\left\vert\kern-0.25ex\left\vert\kern-0.25ex\left\vert
#1 
    \right\vert\kern-0.25ex\right\vert\kern-0.25ex\right\vert}}

\newcommand{\eps}{\varepsilon}
\newcommand{\diag}{\mathrm{Diag}}
\newcommand{\perm}{\mathrm{Perm}}
	
\newcommand{\indep}{\perp \!\!\! \perp}
\newcommand{\indistribution}{\stackrel{\mc{D}}{=}}
\usepackage{enumitem}

\newcommand{\ourtitle}{Function Classes for Identifiable Nonlinear Independent Component Analysis}
\title{\ourtitle}

\author{Simon Buchholz, Michel Besserve \& Bernhard Sch\"olkopf \\

}

\date{%
   Max Planck Institute for Intelligent Systems,  T\"ubingen\\[2ex]
    \today
    }

\begin{document}

\maketitle


\begin{abstract}
Unsupervised learning of latent variable models (LVMs) is widely used to represent data in machine learning. When such models reflect the ground truth factors and the mechanisms mapping them to observations, there is reason to expect that they allow generalization in downstream tasks. It is however well known that such \textit{identifiability} guaranties are typically not achievable without putting constraints on the model class. This is notably the case for nonlinear Independent Component Analysis, in which the LVM maps statistically independent variables to observations via a deterministic nonlinear function. Several families of \textit{spurious solutions} fitting perfectly the data, but that \textit{do not} correspond to the ground truth factors can be constructed in generic settings. However, recent work suggests that constraining the function class of such models may promote identifiability. Specifically, function classes with constraints on their partial derivatives, gathered in the Jacobian matrix, have been proposed, such as orthogonal coordinate transformations (OCT), which impose orthogonality of the Jacobian columns. In the present work, we prove that a subclass of these transformations, conformal maps, is identifiable and provide novel theoretical results suggesting that OCTs have properties that prevent families of spurious solutions to spoil identifiability in a generic setting. 
\end{abstract}

\section{Introduction}\label{sec:intro}


Unsupervised representation learning methods can fit Latent Variables Models (LVM) to complex real world data. While those latent representations allow to create realistic novel samples or represent the data in a compact way \cite{kingma_auto-encoding_2014,goodfellow_generative_2014}, they are a priori not related to the underlying ground truth generative factors of the data. It is highly desirable to recover the true underlying source distribution because those are expected to help with various downstream tasks, e.g., out of distribution generalization \cite{review_representation, bengio_representation}.

 One principled framework for representation learning is Independent Component Analysis 
 (ICA) where one tries to recover unobserved sources $s\in \R^d$ from observations $x=f(s)$
 and one assumes that the components $s_i$ are independent. 
 An important result is that for linear functions $f$ it is possible to
 recover $s$ from observations $x$ up to certain symmetries, i.e., the model is \textit{identifiable} \cite{comon}.
 In contrast, for general non-linear models $f$ is highly non-identifiable \cite{nonlinear_ica}. 
This has important consequences for representation learning, in particular the learning of disentangled representations is also unidentifiable without some access to the underlying sources  \cite{locatello}. 
Notably, this makes theoretical analysis of a large body of methods (see, e.g., \cite{higgins, factor_vae, f_statistic_vae}) that enforces disentanglement difficult. 

Several additional assumptions were suggested  to make the ICA problem identifiable. Broadly, there are two directions. First, some works imposed additional or different restrictions on the 
distribution of the sources. 
One line of research adds temporal structure by considering time series data  \cite{kernel_blind_source, time_contrastive, temporally_ica}. More recently, \citet{ica_auxiliary} proposed to introduce an observed auxiliary 
variable $u$, e.g., a class label, such that the source distribution has independent components conditional on the auxiliary variable. They show that under suitable assumptions on the distribution of $u$ and $s$ arbitrary nonlinear mixing functions can be identified. 
Several recent works extended this approach \cite{ica_vae, gin, yang2022nonlinear}.

Another possibility is to restrict the class of admissible functions by considering more flexible classes than just linear functions, but not allowing arbitrary non-linear functions. The general aim of this approach is to find sufficiently ``small'' function classes such that  
ICA is identifiable within this class while making them as large as possible to allow
flexible representation of complex data and being applicable to real world problems.
So far results in this direction are rather limited. It was shown that the post-nonlinear model is identifiable \cite{post_nonlinear}. 
Moreover, it has been shown that ICA with conformal maps in dimension $2$
is almost identifiable \cite{nonlinear_ica}.
More recently, identifiability of volume preserving transformations 
was investigated in the auxiliary variable case in \cite{yang2022nonlinear} (combining the two possible restrictions) and identifiability based on sparsity of the mixing function was studied in \cite{zheng_neurips}.

In this work, we extend the previous works by proving new identifiability results for unconditional ICA. Our main focus is conformal maps (i.e., maps that locally preserve angles) and Orthogonal Coordinate Transforms (OCT) 
(i.e., maps satisfying that $Df^\top Df$ is a diagonal matrix where $Df$ denotes the derivative of $f$). 
OCTs, that we will also call \textit{orthogonal maps} for simplicity, were recently introduced in the context of representation learning in \cite{ima}, motivated by the independence of mechanisms assumption 
from the causality literature. 
The main focus of this work is to prove new identifiability and partial identifiability results for this class of functions. 
Our main contributions are the following.


\begin{itemize}
\item We prove that ICA with conformal maps is identifiable in $d\geq 3$ and 
extend the earlier results in dimension 2 (Theorems~\ref{th:conformal}-\ref{th:conf2}). 
\item We define a new notion of local identifiability (Definition~\ref{def:loc}) and
prove that  ICA with orthogonal maps is locally identifiable (Theorem~\ref{th:loc_loc}). 
\item On the contrary we show that ICA with
volume preserving maps is not identifiable not even in the local sense (Theorem~\ref{th:volume}).
\item We introduce new tools to the ICA field: our results are based on connections to rigidity 
theory, restricting the global structure of functions based on local 
restrictions. Moreover, in contrast to most earlier results that argue locally using results from linear algebra we exploit the global structure of partial differential equations related to the identifiability problem. 
\end{itemize}


The remainder of this paper is structured as follows. In Section~\ref{sec:setting} we introduce the general setting of ICA and identifiability. 
Then we discuss our results for different classes of nonlinear maps. We consider conformal (Section~\ref{sec:conformal}), 
orthogonal (Section~\ref{sec:ima}) and volume preserving  (Section~\ref{sec:volume}) maps. An overview of our results can be found in 
Table~\ref{tab:results}. Finally, in  Section~\ref{sec:rigidity}, we discuss the relation of identifiability of ICA to the rigidity of the considered function class.


\section{Setting}\label{sec:setting}


Independent component analysis investigates the problem of identifying underlying sources when observing a mixture of them. 
We will consider the following general setting: there
exists some random hidden vector of sources $s\in \R^d$ 
and the observed data is generated by


\begin{align}\label{eq:general_data}
x = f(s),\qquad  p_s(s)=\prod_{i=1}^d p_i(s_i)=\P
\end{align}


where $f:\R^d\to \R^{d}
$ is a smooth invertible function. 
The condition on $s$ means that its coordinates (often referred to as factors of variation) 
are independent. 
Formally this means that the distribution of  $s$ which we denote by $\P$ satisfies
$\P\in \mc{M}_1(\R)^{\otimes n}$ where $\mc{M}_1(\R)$ denotes the probability measures on $\R$.
The goal of ICA is  to find an unmixing function $g:\R^{d}\to\R^d$ such that $g(x)$ has independent components.
Ideally, this should recover the true underlying factors of variation
and achieve \textit{Blind Source Separation}
(BSS), i.e., $g=f^{-1}$ up to certain symmetries.
Identification of the true generative factors of variations of an observed data distribution is of interest also since these provide a causal and interventional understanding of the data.

An important observation was that, in the generality stated above, identification of $s$ is not possible. In \cite{nonlinear_ica} two general constructions of
spurious solutions were given: the well known Darmois construction and 
a construction based on measure preserving transformations. 
The latter one is closer to our work here and we will discuss those in more detail in Section~\ref{sec:ima} and Appendix~\ref{app:spurious}. In a nutshell it is based on the observation that for measures $\P$ with smooth density one can construct smooth \textit{Measure Preserving Transformations} (MPT), $m:\R^d\to \R^d$
(that mix the different coordinates), i.e., maps that leave $\P$ invariant, such that $m(s)\indistribution s$ if $s\sim \P$.\footnote{We use the notation $X \indistribution Y$
to indicate that the two random variables $X$ and $Y$ follow the same distribution} This implies that all
functions  $(f\circ m)^{-1}$  recover independent sources since $(f\circ m)^{-1}(x)\indistribution s$ making
BSS impossible.

Thus it is a natural question whether additional assumptions on the mixing function
$f$ or distribution of $s$ allow us to identify $f$.
Let us define a framework for identifiability.
We assume data is generated according to \eqref{eq:general_data} 
where $f$ belongs to some function class $\mc{F}(A,B)$ of invertible functions $A\to B$, 
which we will always assume to be diffeomorphisms.\footnote{A diffeomorphism is a differentiable bijective map with differentiable inverse.} and 
$\P$ satisfies $\P\in\mc{P}$ for some set of probability distributions $\mc{P}\subset
\mc{M}_1(\R)^{\otimes d}$.
Finally, let $\mc{S}$ be a group of transformations $g:\R^d\to \R^d$ that encodes the allowed symmetries up to which the sources can be identified as follows. The function class will be simply denoted $\mc{F}$ when domain/codomain information is irrelevant.
\begin{figure}
\centering
    \includegraphics[width=.7\textwidth]{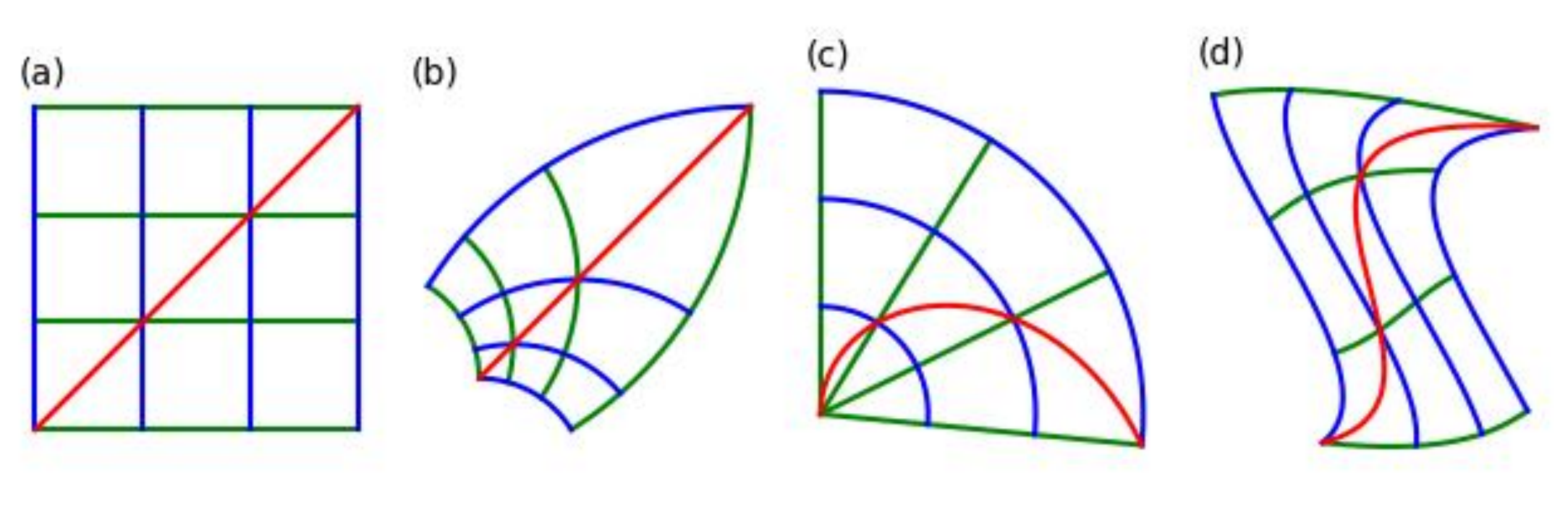}
    \caption{Illustration of the considered function classes. (a) shows a
    standard coordinate frame, (b) a conformal map applied to this frame which preserves angles, (c) an orthogonal map (polar coordinates) that preserve the orthogonality of lines parallel to the coordinate axes but not all angles (see red line), (d) a volume preserving map.}
    \label{fig:maps}
\end{figure}
\begin{definition}(Identifiability) \label{def:ident}
We say that independent component analysis in $(\mc{F},\mc{P})$ is identifiable
up to $\mc{S}$ if for functions $f, f'\in \mc{F}$ and distributions $\P, \P'\in 
\mc{P}$ the relation
\begin{align}\label{eq:def_equal_dist}
f(s) \indistribution f'(s')\quad\text{where } s\sim \P \text{ and } s'\sim\P'
\end{align}
implies that there is  $h\in \mc{S}$ such that $h=f'^{-1}\circ f$
on the support of $\P$.
\end{definition}
Note that we require the identity $h=f'^{-1}\circ f$ only to hold
on the support of $\P$ because, for complex classes $\mc{F}$,  
there is in general no unique extension of $f$ beyond the support of $\P$
and without data the extension cannot be identified. 
We do not always make this explicit in the following.
Put differently, identifiability means that given observations
of $x=f(s)$ and knowledge of $(\mc{F},\mc{P})$, 
we can find $g$ such that $g\circ f\in \mc{S}$, in particular  the reconstructed 
sources $s'=g(x)$ and the true sources $s$ are related by 
a symmetry transformation in $\mc{S}$.
We discuss in Appendix~\ref{app:spurious} how to identify the set $\mc{S}$ and how spurious solutions to the identification problem can be constructed. 
In the following, it will be convenient to use the notation $f_\ast \P$
which denotes the push-forward of the measure $\P$ along the function $f$.
We refer to Appendix~\ref{app:math} for a formal definition,  
but we note here that the distribution of $f(s)$ equals $f_\ast \P$
whenever $s\sim\P$. Therefore, \eqref{eq:def_equal_dist} can be equivalently written as $f_\ast\P=f'_\ast\P'$. 

We illustrate Definition~\ref{def:ident} through the well known example
of linear maps
\begin{align}\label{eq:flin}
\mc{F}_{\mathrm{lin}}=\{f:\R^d\to\R^d\, : \text{$f$ is linear and invertible}\},
\end{align}
i.e., $x = As$ for some invertible matrix $A\in \mathbb{R}^{n\times n}$.
We further define
\begin{align}\label{eq:plin}
\mc{P}_\lin &= \{\P\in \mc{M}_1(\R)^{\otimes d}:\, \text{at most one component of $\P$ is Gaussian}\},\\ \label{eq:slin}
\mc{S}_{\lin}&=\{ P\Lambda \, : \text{$P$ is a permutation matrix and $\Lambda$ is a diagonal matrix}\}.
\end{align}
It is easy to check that $\mc{S}$ is a group.
Then the following identifiability result for $\mc{F}_\lin$ is well known.
\begin{theorem}\label{th:linear}(Theorem~11 in \cite{comon})
The pair $(\mc{F}_\lin,\P_\lin)$ is identifiable up to $\mc{S}_\lin$.
\end{theorem}
This result is optimal as the ordering and scale of the $s_i$ is unidentifiable
and the restriction to at most one Gaussian component is required to avoid linear MPTs of multivariate Gaussians. We provide a proof of this result 
in Appendix~\ref{app:linear} as this serves as a preparation for the more involved Theorem~\ref{th:conformal} below.

{
An important observation which was  also made in \cite{when_disentangle} is that with minor differences one can also consider the case where $f:\R^d\to M$ maps to a $d$-dimensional Riemannian manifold $M$.
An important example for this setting  is the case where $M\subset \R^n$ is a submanifold of a higher dimensional Euclidean space. This covers the standard setting of unsupervised representation learning where high dimensional observations
(often images) are created from low dimensional factors of variation mirroring the well known manifold hypothesis \cite{Tenenbaum2000}.
Note that this setting essentially covers the case of undercomplete ICA, where we consider $f:\R^d \to \R^n$ with $n> d$. The only difference is that we assume that we already know the submanifold $M$ that $f$ maps to. This manifold can, however, be identified from the observations $x=f(s)$ under minor regularity assumption on $f$ and the support of the data distribution. To avoid technical difficulties we assume  that the manifold is already known.
Note that we restrict our attention to the case where the factors of variations
are parametrized by a Euclidean space. An extension to product manifolds and a combination with the approach in \cite{higgins_disentangle} is an interesting question left to future work. 
}

In the next sections we discuss our results on identifiability of ICA for different
function classes. An illustration of the considered classes can be found in Figure~\ref{fig:maps}.
{
They are all characterized by a local condition on their gradient.
Previously, in \cite{when_disentangle} it was shown that the function class of local isometries is identifiable. Local isometries have been used frequently in machine learning, and more specifically in representation learning \cite{lin_emb,Tenenbaum2000,donoho_grimes}. Our main results consider two generalizations of these function classes, conformal maps and orthogonal coordinate transformations. Conformal maps  preserve angles locally and have been used in computer vision \cite{conf_2d,conf_3d,conf_med}. For $d=2$, conformal maps essentially consist of all biholomorphic mappings of simply connected open domains of the complex plane, and thus constitute a ``large'', non-parametric family, as a consequence of the Riemann Mapping Theorem \cite{osgood1900existence}. For $d\geq 3$, Liouville's Theorem implies that this class contains relatively ``few'' functions in fixed dimension (i.e., mapping $\R^d\to\R^d$), in the sense that it is a parametric family, with parameter dimension quadratic in $d$ (see Theorem~\ref{th:liouville}). However, this is a rich class when the target space is higher dimensional than the domain (i.e., $\R^d\to \R^{n}$, $d<n$). 
OCT are an even more general class 
that were motivated based on the principle of \textit{independent causal mechanisms} in \cite{ima}. Notably, it contains all conformal maps precomposed with nonlinear entrywise reparametrizations of the source components (see Corollary~\ref{co:rescale}). It is however much larger, as one can for example concatenate arbitrary functions from the large family of 2d conformal mappings to obtain higher dimensional OCTs. 
Moreover, many works showed that training VAEs promotes orthogonality of
the columns of the input Jacobian \cite{rolinek, zietlow, kumarpoole} and this has been empirically shown to be a good inductive bias for disentanglement. Indeed, these algorithms are widely used in representation learning and often recover semantically meaningful representations \cite{kumar_latent_concepts,chen_isolating_disentangle,factor_vae,higgins}. 
}

\begin{table}
\centering
        \caption{Overview of new identifiability results. Note that \emph{Identifiable}
         implies \emph{Locally identifiable}
        and if \emph{Locally identifiable} does not hold neither of the other two properties can hold.}
    \label{tab:results}
       \begin{tabular}{c||>{\centering} p{1cm}c|>{\centering} p{1cm} c |>{\centering} p{1cm}c}
        Function class &
        \multicolumn{2}{| >{\centering} p{2.7cm}|}{Identifiable (Def.~\ref{def:ident})}& \multicolumn{2}{| >{\centering} p{2.7cm}|}{Locally identifiable (Def.~\ref{def:loc})}& 
        \multicolumn{2}{>{\centering} p{2.7cm}}{ Gaussian only spurious solution}\\
        \hline \hline
        Linear     &  \cmark&  & \cmark & & \cmark& \\
        Conformal & \cmark & (Thm.~\ref{th:conformal} \& \ref{th:conf2}) & \cmark & & \cmark& \\
        Orthogonal & \textbf{?}&  & \cmark & (Thm.~\ref{th:loc_loc}) & \xmark
        & (Prop.~\ref{prop:non-uniqueness})\\
        Volume preserving & \xmark & & \xmark &(Thm.~\ref{th:volume})& \xmark &\\
        General nonlinear & \xmark  & & \xmark & (Lemma~\ref{le:const_rotations}) & \xmark &
        \end{tabular}
\end{table}


\section{Results for conformal maps}\label{sec:conformal}

Our first main result is an extension of Theorem~\ref{th:linear} to conformal maps.
A conformal map is a map that locally preserves angles, i.e. 
locally it looks like a scaled rotation.
It can be shown that this is equivalent to the following definition.
 \begin{definition}\label{def:conformal}(Conformal map)
  We define for domains $\Omega\subset\R^d$ the set of conformal maps by $\mc{F}_\conf=\{f\in C^1(\Omega,\R^d) : \, Df(x) = \lambda(x)O(x)\}$ where $\lambda:\Omega\to\R\setminus \{0\}$ is a scalar function and $O:\Omega\to \mathrm{O}(d)$ is a map to orthogonal matrices (i.e., $O(x)^{-1}=O(x)^\top$).
  \end{definition}
 All our results also hold for the more general class of conformal maps $f:\Omega\to M$ where $M$ is a Riemannian manifold. The complete definition can be found in Appendix~\ref{app:ext}.
For convenience we define signed permutation matrices by 
\begin{align}
\perm_\pm(d)=
\{P\in \R^{d\times d}:\, \text{$Q\in \R^{d\times d}$, such that $Q_{ij}=|P_{ij}|$, is a permutation matrix}\},
\end{align}
i.e. the set of matrices whose entry-wise absolute value is a permutation.
Later we will also use the notation $\diag(d)$ and $\perm(d)$
for $d\times d$ diagonal and permutation matrices, respectively.
We define
\begin{align}
\mc{S}_{\conf}&=\{x\to \kappa Px+a\; \text{ where $P\in \perm_\pm(d)$, $ a\in \R^d$, $\kappa\in \R\setminus\{0\}$}\}
\end{align}
and 
\begin{align}
\begin{split}
\mc{P}_\conf&=\mc{P}_1^{\otimes n} \cap \mc{P}_\lin, \qquad\text{where}
\\
\mc{P}_1 &= \{\mu \in \mc{M}_1(\R),\; \text{
there is  $\emptyset \neq O\subset \R$ open, s.t.\  $\mu$ has positive $C^2$ density  on $O$}  \}.
\end{split}
\end{align}  
While this condition might appear a bit technical it actually only rules out
pathological cases like the cantor measure or densities which are nowhere differentiable and probably it could be relaxed further. 
In particular $\mc{P}_1$ contains all probability measures with piecewise smooth densities.
Then the following identifiability for conformal maps in dimension $d > 2$ holds.
\begin{theorem}\label{th:conformal}
For $d>2$, ICA with respect to the pair $(\mc{F}_\conf, \mc{P}_\conf)$ is identifiable
up to $\mc{S}_{\conf}$.
\end{theorem}
This means that we can identify conformal maps up to three symmetries, namely constant shifts
of the distributions, rescaling of all coordinates by the same constant factor, and
permutations of the coordinates. The proof is in 
Appendix~\ref{app:conformal}.
The main ingredient in the proof is that conformal maps in dimension $d>2$
are very rigid and can be characterized explicitly, as we will discuss in Section~\ref{sec:rigidity}.

We remark that it might be more natural to not fix the scale of the sources and allow arbitrary coordinate-wise rescalings. The result can be easily extended to accommodate this. We define 
\begin{align}
   \begin{split}\label{eq:def_repara}
 \mc{S}_{\rep}&= \{ g:\R^d\to\R^d|\, g=P\circ h\, \text{where $P\in \perm_\pm(d)$ and
$h:\R^d\to\R^d$ with }
\\
&\qquad\quad \text{ $h(x)=(h_1(x_1),\ldots, h_d(x_d))^\top$
for some $h_i\in C^1(\R,\R)$ with $h_i'>0$}\}.
\end{split}
\end{align}
It is easy to see that $\mc{S}_{\rep}$ is a group. We define the class of 
reparameterized conformal maps by $\mc{F}_{\sconf}=\{f\circ h\,| f\in \mc{F}_\conf, h\in \mc{S}_\rep\}\cap C^3(\Omega,M)$ and then get the following Corollary.
\begin{corollary}\label{co:rescale}
For $d>2$, ICA with respect to the pair $(\mc{F}_{\sconf},\mc{P}_\conf)$
is identifiable up to $\mc{S}_{\rep}$ if we assume in addition that the observational distribution cannot be expressed as $f_\ast \P$ for some $f\in \mc{F}_\conf$ and $\P\in \mc{M}_1(\R)^{\otimes d}$ which has at least two Gaussian components.
\end{corollary}
The additional restriction on the observational distribution is clearly necessary to exclude the non-identifiability of Gaussian distributions. 

For dimension $d=2$ it was shown \cite{nonlinear_ica} that conformal maps
can be identified up to a rotation  when fixing one point of the conformal map (setting
$f(0)=0$).
The authors also claim, without proof, that the remaining ambiguity can be removed for typical probability distributions.
We extend their result by removing the condition that one point is fixed and 
prove full identifiability with a minor assumption on the involved densities.
We define the following set of probability measures on $\R^2$
\begin{align}
\begin{split}
\mc{P}_{\conf2}& = \{\P=\P_1\otimes \P_2\in \mc{M}(\R)^2 |\; \text{s.t.\  $\mathrm{supp}(\P_i)$ is
a bounded interval $I_i$ and}
\\
&\qquad \qquad \text{$\P_i$ has density bounded above and below on $I_i$}  \}
\end{split}
\end{align}
Then we get the following result.
\begin{theorem}\label{th:conf2}
For $d=2$, ICA with respect to the pair $(\mc{F}_\conf, \mc{P}_{\conf 2})$ is
identifiable up to $\mc{S}_{\conf}$.
\end{theorem}
This means that we can identify conformal maps on compact domains in dimension 2 up to 
shifts, permutations of coordinates, and scale. Note that we can also identify conformal maps if $\P$ has full support $\R^2$ using the
same proof as for $d>2$ (see Lemma~\ref{le:inversion}
in the supplement) and an extension as in Corollary~\ref{co:rescale} is possible. The proof of this result is in Appendix~\ref{app:conformal}.
It relies on the Schwarz-Christoffel mapping that provides a formula for conformal maps from  the upper half plane or the unit disc to polygons.


\section{Results for orthogonal maps}\label{sec:ima}


Recently, in \cite{ima}, the more general class of OCTs was considered in the context of ICA. They referred to
orthogonal coordinates as IMA maps, referencing to independent mechanisms.
This nomenclature was motivated 
by the causality literature and we refer to their paper for an extensive motivation and further results. As we focus on theoretical results for this function class
we stick to the more common term of OCTs.
Orthogonal coordinate transformations are defined as the set of functions whose derivative have orthogonal columns, i.e., 
the vectors $\partial_i f$ and $\partial_j f$ are orthogonal for $i\neq j$.
\begin{definition}(OCT maps)
\label{def:ima}
 We define for domains $\Omega\subset\R^d$ the set of OCT maps (orthogonal coordinates) by $\mc{F}_\ima=\{f\in C^1(\Omega,\R^d) : \, Df(x)^\top Df(x)\in \diag(d)\}$.
\end{definition}
\paragraph{OCTs constitute a rich class of functions.}
The study of OCTs has a long history and already in the 19th century the structure of all OCTs defined in a neighbourhood of a point were characterized \cite{darboux1896leccons, Cartan1925}.
Later, also the set of global orthogonal coordinate systems on $\R^d$ was characterized \cite{coordinate_systems}. As those results are not easily accesible we will provide here a simple argument showing that OCTs constitute a rich class of functions.
We first note that $\mc{F}_\ima$  contains the above $\mc{F}_\sconf$, as functions in the later class have a $Df(x)$ that takes the form of a Jacobian of a conformal map whose columns are rescaled by derivatives of the entry-wise reparameterizations, such that they remain orthogonal. However, $\mc{F}_\ima$ is much bigger than $\mc{F}_\sconf$. For example, take a $n$-tuple $(f^1,...,f^n)$ of arbitrary injective 2D conformal maps 
\[
f^k:\, \Omega_k \to \R^2 \in \mc{F}_\conf,\, k= 1 ... n ,\, \Omega_k \subset \R^2
\]
and build the ``concatenated'' map 
\begin{align*}
f_{conc}:\, & \Omega_1 \times \dots \times \Omega_n \to  \R^{2n}\\
            & s \quad \mapsto (f^1_1(s_1,s_2),f^1_2(s_1,s_2),...,f^k_1(s_{2k-1},s_{2k}),f^k_2(s_{2k-1},s_{2k}),...,f^n_1(s_{2n-1},s_{2n}),f^n_2(s_{2n-1},s_{2n}) )^\top\,.
\end{align*}
The Jacobian of $f_{conc}$ is block diagonal, such that columns associated to different diagonal blocks are obviously orthogonal, and columns pertaining to the same k-th diagonal block are orthogonal by conformality of $f_k$. With such a construction, that we can also further post-compose with transformations in $\mc{F}_\conf$ on $\R^{2n}$, we can thus build a large non-pararametric subclass of $\mc{F}_\ima$ on $\R^{2n}$. This construction can be easily adapted to the case of odd dimensions.

\paragraph{Setting for identifiability with OCTs.} OCTs can also be generalized to maps whose target is a $d$-dimensional manifold (see definition in Appendix~\ref{app:ext}), and the following results will also apply to such case.  
First, we note that we can only hope to identify a mechanism $f\in \mc{F}_\ima$ up to
coordinate-wise transformations and permutations, i.e., maps in $\mc{S}_\rep$.
Indeed, if $f\in \mc{F}_\ima$
and $g\in \mc{S}_\rep$ then $f\circ g\in \mc{F}_\ima$. Thus, in particular
$f_\ast \P=(f\circ g)_\ast (g^{-1})_\ast\P$. This implies that given observations from $f_\ast \P$ we can identify $f$ and $\P$ only up to $g\in \mc{S}_\rep$.
More precisely, for any (sufficiently smooth) $\P'$ there is
$f'$ such that $f_\ast \P=f'_\ast \P'$ where we pick $g$ such that $\P'=g^{-1}_\ast\P$.
\footnote{This is possible if both distributions have compact connected support where they have a smooth positive density. We ignore difficulties associated with unbounded support or non-regular measures here}

As the distribution of the $s_i$ is not identifiable, we  
map it to a fixed reference distribution that we choose to be the uniform distribution on $(0,1)^d$. 
We introduce the shorthand
$\cube=(0,1)^d$ for the standard open unit cube (exclusion of the boundary will be important for our result) and denote
 by $\nu$ the uniform (Lebesgue) measure on $\cube$. For fixed base measure $\nu$ the symmetry group is reduced to permutations and reflections, i.e., maps in $P\in \perm_\pm(d)$.
 
 We conjecture that for 'typical' pairs $(f,\,\P)\in \mc{F}_{\ima}\times \mc{P}_\ima$, ICA is identifiable with respect to $\mc{S}_{\rep}$
(with a suitable definition of $\mc{P}_{\ima}$, e.g., $\mc{P}_\ima=\mc{P}_\conf$).
However, we leave a precise statement for future work. Below we will show a weaker notion of identifiability for OCTs, but, before that, we first exhibit exceptional classes of spurious solutions for ICA with OCTs.
 
 
\paragraph{Spurious solutions for ICA with OCTs.} 
We first note that just as for linear maps and conformal maps (see Thm.~\ref{th:linear}) Gaussian distributions can hamper identification. 
This is because arbitrary measures with a factorized density can be pushed forward into multivariate Gaussians using a suitable coordinate-wise transformation. 
\begin{fact}
Let $\P$ be a probability measure on $\R$ with bounded density $p$ 
and cumulative distribution function $F_\P$. Denote the 
cumulative distribution function of the standard normal by $F_\mc{N}$.
Let $h_\P=F^{-1}_\mc{N}\circ F_\P$. Then $(h_\P)_\ast \P$ has a standard normal distribution.
\end{fact}
This implies that if $\P=\P_1\otimes\ldots\otimes\P_d$ and $x=f(s)=(h_{\P_1}(s_1),\ldots,h_{\P_d}(s_d))$
then $f_\ast \P$ follows a standard normal distribution. 
In particular, for every $A\in \mathrm{O}(d)$ the distribution
of $Ax$ is standard normal and its components are independent.
Note that, in contrast to conformal and linear maps, it is not sufficient to exclude Gaussian source distributions: due to the flexibility of the function class, we have to exclude that the pair $(f,\P)$ has a Gaussian observational distribution $f_\ast \P$. 
Next we show that a more general construction using OCTs is possible.
\begin{proposition}\label{prop:non-uniqueness}
Let $\P$ be a rotation invariant distribution on $\R^d$
with smooth density.  
Then there is a smooth and invertible (on its image) function $f:\cube\to \R^d$ with $f\in \mc{F}_{\ima}$
such that $f_\ast \nu = \P$.
\end{proposition}


The proof of this result can be found in Appendix~\ref{app:ima}.
The main idea in the proof is that $d$-dimensional 
polar coordinates do the trick up to coordinate-wise rescaling.
This proposition implies that the entire family $\{f_R=R\circ f |\, R\in O(n)\}$
satisfies $(f_R)_\ast \nu = f_\ast \nu$
and, by definition of $\mc{F}_\ima$ we have $f_R\in \mc{F}_\ima$
because $f\in \mc{F}_\ima$ implies $R\circ f\in \mc{F}_\ima$.
In particular all inverses $g_R=f_R^{-1}$ recover independent sources
in the sense that $(g_R)_\ast f_\ast\nu=\nu$
and BSS is not possible in a meaningful way in this (special) case.
The construction 
in Proposition~\ref{prop:non-uniqueness} gives spurious solutions for substantially more observational distributions than just the Gaussian (this is indicated in Table~\ref{tab:results}).
Nevertheless we do not view this as a general obstacle to identifiability results for
OCTs for two reasons. Firstly the spurious solutions  only apply to carefully chosen pairs of function and base measure such that we obtain a still very non-generic (radial) observational distribution. Moreover, the function constructed in Proposition~\ref{prop:non-uniqueness} cannot be extended to $\overline{\cube}=[0,1]^d$ such that it remains invertible 
(in the language of differential geometry this means that $f$ is only an immersion not an embedding of a submanifold). In this sense the main message of Proposition~\ref{prop:non-uniqueness} is that an identifiability result for $\mc{F}_\ima$ needs to contain assumptions ruling out those spurious solutions (just like Gaussians are excluded for linear ICA).


\paragraph{Local Stability of OCTs.}
\begin{figure}
\begin{minipage}{0.45\textwidth}
    \centering
    \includegraphics[width=\textwidth]{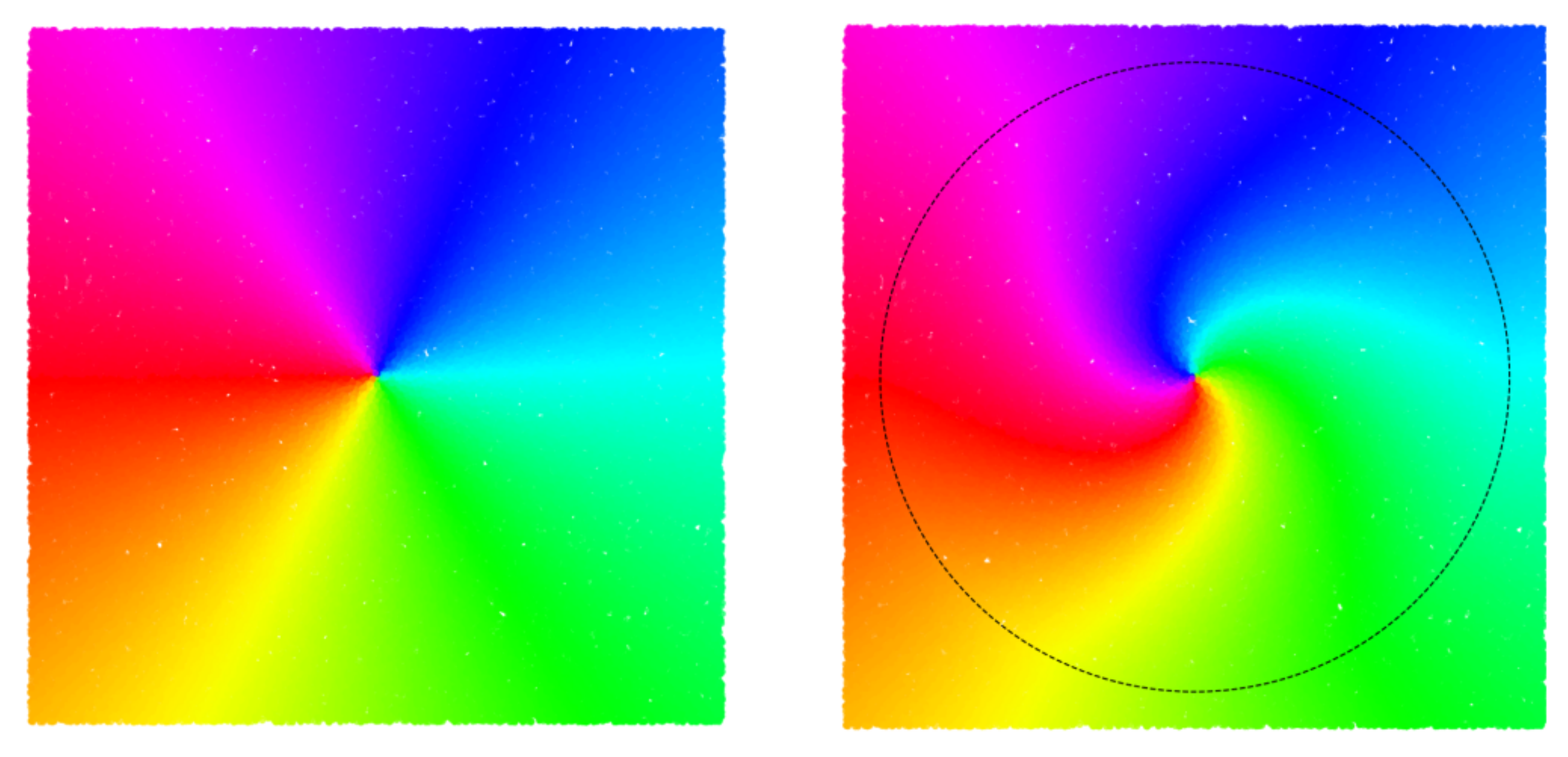}
    \caption{Illustration of radius dependent rotations as defined in Lemma~\ref{le:const_rotations}. The left figure shows the initial sources.
    In the right figure a radius dependent volume preserving transformation was applied (see Appendix~\ref{app:spurious}).}
    \label{fig:flow}
\end{minipage}
\begin{minipage}{.05\textwidth}
    \mbox{}
\end{minipage}
\begin{minipage}{0.45\textwidth}
    \centering
    \includegraphics[width=\textwidth]{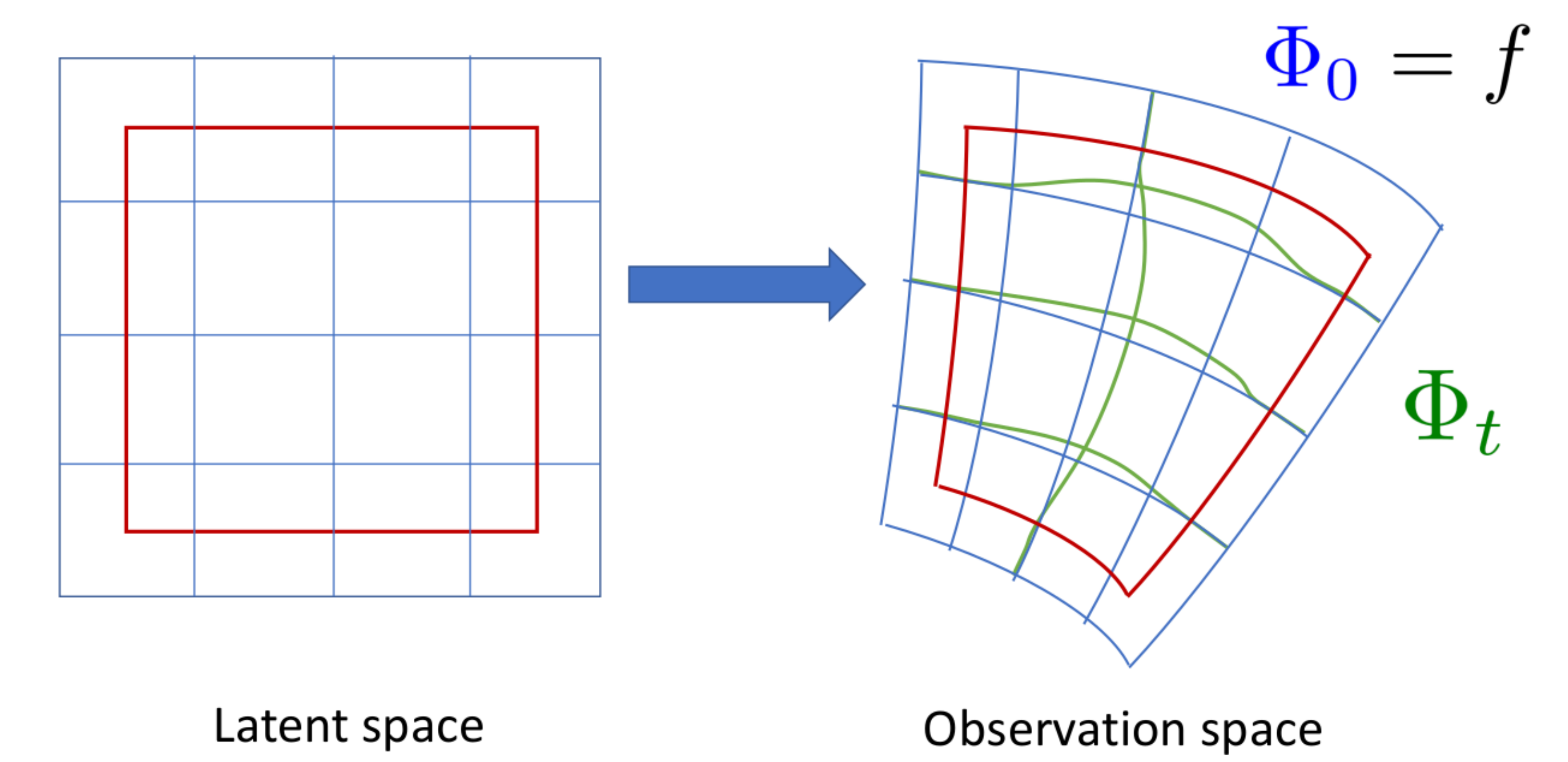}
    \caption{Smooth invariant deformations. The blue grid indicates the transformation $f$, while the green grid shows the deformed map $\Phi_t$.  Outside of the red box $\Phi_t=f_t$ holds (see Definition~\ref{def:loc}).}
    \label{fig:deform}
\end{minipage}
    \end{figure}
We now give partial results towards identifiability of OCTs. 
While we do not prove general identifiability for this class, we demonstrate their \textit{local rigidity}: OCTs cannot  be continuously deformed to obtain spurious solutions. This is in stark contrast to the general nonlinear case, 
which we will discuss for comparison below. 
Actually, we will show the following slightly stronger statement. Suppose we know some initial 
data generating mechanism $f_0$ because we have, e.g., access to 
samples $(s, f_0(s))$. Now we assume that 
the mixing $f_t$ depends smoothly on some parameter $t$  which could be, e.g., time
or an environment.
Then we can identify $f_t$ if $f_t\in \mc{F}_{\ima}$ for all $t$
and given access to samples $f_t(s)$. This would not be true if $f_t$'s would be 
unconstrained nonlinear mixing functions. We now make these statements precise.
We first define smooth deformations of a mixing function.
{
\begin{definition}(Smooth invariant deformations)
Let $\mc{F}$ be some function class.
Consider a family of differentiable transformations $\Phi \in C^1((-T, T)\times \R^d, M)$ for some $T>0$ and a smooth, $d$-dimensional manifold $M$ such that $\Phi_t=\Phi(t,\cdot)$ is a diffeomorphism onto its image. 
We call $\Phi_t$ a smooth invariant deformation if $\Phi_t(\cdot)\in \mc{F}$
for all $t$.
\end{definition}


An illustration of this definition can be found in Figure~\ref{fig:deform}. Based on invariant deformations we can now define a local
identifiability property of ICA in a given function class.
\begin{definition}(Local identifiability of ICA).\label{def:loc} Consider a function class $\mc{F}$.
Let $f_t$ be a smooth invariant deformation in $\mc{F}$ that is analytic in $t$.
Let $\Phi_t$ be another smooth invariant deformation in $\mc{F}$
analytic in $t$
such that $\Phi_0=f_0$ and  $(\Phi_t)_\ast \nu=(f_t)_\ast\nu$ for all $t$.
Assume that there is $\varepsilon>0$
such that $\Phi_t(s)=f_t(s)$ if  $\mathrm{dist}(s, \partial C_d)<\varepsilon$.
Then we say that ICA in $\mc{F}$ is \textit{locally identifiable at $(f_t,\nu)$} if  these assumptions imply
$\Phi_t=f_t$ for all $t$.
We call $\mc{F}$ \textit{locally identifiable} if it is locally identifiable at $(f_t,\nu)$
for all analytic local deformations $f_t$.
\end{definition}


Local identifiability of a function class means that we can identify smooth 
deformations $f_t\in \mc{F}$
from some initial mixing $f_0\in \mc{F}$ 
if $f_0$ is known on the whole latent domain, and the behaviour of $f_t$ close to the domain's boundary 
is known for all $t$. 
This can be interpreted as plain identifiability  in the concept drift setting 
\cite{concept, concept2}. 
Formulated differently, it means that an adversary cannot smoothly deform  the function $f_0$ in a subset whose boundary is away from the boundary of the domain, such that the outcome is ambiguous given the resulting observational distributions. An experimental illustration of this setting and Theorem~\ref{th:loc_loc} below can be found in Appendix~\ref{app:ill}. 
The notion of locality in Definition~\ref{def:loc} should be understood as ``non-global'' and notably does not imply restrictions to a small neighborhood, as local properties often do. The non-globality manifests itself in two ways: we consider smooth transformations 
of the ground truth, i.e., small changes of the data generating function $f_0$ and in addition we assume that the changes are not everywhere in $s$, i.e., sources $s$ close to the boundary are kept invariant. An extension to other source distributions beyond fixed $\nu$ (and possibly changing with $t$) is possible but not necessary in the context of OCTs as explained above. We can show that local identifiability holds true in $\mc{F}_\ima$.
\begin{theorem}\label{th:loc_loc}
The function class $\mc{F}_\ima$ is locally identifiable.
\end{theorem}
}


The proof, in Appendix~\ref{app:ima},
 introduces new tools to the field of ICA. The main idea is to consider the vector field $X$ that generates the deformation $\Phi_t$
and then rewrite the assumption as systems of partial differential equations for $X$.
The proof is then completed by showing that the only solution of this system vanishes.
Let us state one simple consequence of this theorem.
\begin{corollary}\label{co:undeformed}
Let $\Phi_t$ be a smooth analytic invariant deformation 
in $\mc{F}_{\ima}$ such that $(\Phi_t)_\ast \nu = f_\ast \nu$ for all $t$
and there is $\varepsilon>0$ such that $\Phi_t(s)=f(s)$ if $\mathrm{dist}(s,\partial\cube)<\varepsilon$.
Then $\Phi_t=f$ for all $t$.
\end{corollary}
Let us reiterate what this corollary shows: we cannot smoothly and locally transform the function $f$ such that (1) the observational distribution remains invariant, i.e., equal to $f_\ast\nu$,
and (2) the deformed functions remain OCTs. 

At a high level this result suggests that OCTs can be identified 
if we know $f$ close to the boundary of the support of $s$, e.g., by having, in addition to unlabelled data $f(s)$, labelled data $(s,f(s))$ 
for those $s$ where one coordinate $s_i$ is extremal.
Note that we actually do not show this result as there might be further solutions
which are not connected by smooth transformations. 
We expect that those results can be generalized substantially. 
In particular, we conjecture that for ``most'' functions $f$ the boundary condition can be removed thus giving a stronger local 
identifiability result up to the boundary of the support of $s$. 
As a partial result in this direction we prove the following theorem.
\begin{theorem}\label{th:loc_globalalt}
Let $f:\cube\to \R^d$ be given by $f(x)=RDx$, with $R\in \mathrm{O}(d)$ and $D=\diag(\mu_1, \ldots, \mu_d)$ where $\mu_i$ 
are i.i.d.\ samples from a distribution supported on the positive reals
$\R_+$ which has a density.
Suppose that $\Phi_t$  is a smooth invariant deformation in $\mc{F}_{\ima}$ such that $\Phi_0=f$,
$(\Phi_t)_\ast \nu=f_\ast \nu$, and $\Phi_t$ is analytic in $t$.
Then for almost all $\mu_i$ (i.e., {with probability one}) this implies
$\Phi_t=f$ on $\cube$, i.e., 
$\Phi_t$ is constant in time.
\end{theorem}
In Appendix~\ref{app:ima} we show that this theorem follows from a slightly
stronger result stated as Theorem~\ref{th:loc_global_app}
which has a similar proof as  Theorem~\ref{th:loc_loc}. 
We do expect that the conclusion of the Theorem actually holds for all $\mu_i$ 
not just almost all, but we are unable to show this.


\paragraph{Comparison with ICA for general nonlinear functions.}
Let us emphasize that those results are  non-trivial as they establish a 
large difference between 
ICA with generic nonlinear maps and ICA with OCTs.
To clarify this, we state that no result similar to Theorem~\ref{th:loc_loc}
holds  without the assumption that $\Phi_t\in \mc{F}_\ima$. Put differently,
the function class $\mc{F}_{\mathrm{nonlinear}}$ is not locally identifiable.
\begin{fact}\label{fa:non_uniqueness}
Suppose $f:\cube\to \R^d$ is a diffeomorphism on its image.
Then there are uncountably many smooth deformations $\Phi_t$ of $(f, \nu)$ 
such that (i) $(\Phi_t)_\ast\nu=f_\ast\nu$ and (ii) there is an  $\varepsilon>0$  such that $\Phi_t(s)=f(s)$
whenever $\mathrm{dist}(s, \partial\cube)<\varepsilon$. 
\end{fact}
For completeness, we provide a general construction
that is close to our proof of Theorem~\ref{th:loc_loc} in Appendix~\ref{app:spurious} in the supplement. 
A very clear construction for this result was given in \cite{nonlinear_ica}. 
\begin{lemma}[Smoothly varying radius dependent rotations (see \cite{nonlinear_ica})]
\label{le:const_rotations}
Let $R:\R\times \R_+\to \mathrm{O}(d)$ be a smooth function mapping to orthogonal matrices
and let $a\in \cube$. Assume that
$R(t,r)=\mathrm{id}$ for $r\geq \mathrm{dist}(a, \partial \cube)$.
Then the map $s\to h_{R,a}(t,s)= R(|s-a|, t)(s-a) + a$ preserves the uniform measure $\nu$ on $\cube$
for all $t$ so that $f\circ h_{R,a}(t, s)\indistribution f(s)$ for all $t$ if $s$ is distributed according to $\nu$.
\end{lemma}
An illustration of this construction is shown in Figure~\ref{fig:flow} (see App.~\ref{app:spurious} for details).
Clearly, by concatenation this allows us to create a vast family of spurious solutions.
Note that those solutions are excluded when restricting to OCTs
which is a corollary of Theorem~\ref{th:loc_loc}.
\begin{corollary}
Suppose $f\in \mc{F}_{\ima}$.
Let $\Phi_t$ be the smooth invariant deformation defined by 
$\Phi_t = f\circ h_{R,a}(t, \cdot)$ where $h_{R,a}$ is as in Lemma \ref{le:const_rotations}.
If $\Phi_t\in \mc{F}_\ima$ for all $t$ this implies that $h_{R,a}(t,s)=s$
and $\Phi_t=f$ for all $t$.
\end{corollary}

We now summarize our view on the results of this section informally (we do not claim that the statements below regarding (infinite dimensional) manifolds can be made rigorous). For a given data generating mechanism $(f_0,\nu)$ we expect that typically the set of all solutions $M_\ima=\{f\in \mc{F}_\ima|\; f_\ast\nu=(f_0)_\ast\nu\}\subset \mc{F}_\ima$ is
a zero dimensional submanifold, i.e., consists of isolated spurious solutions and we prove this when fixing the boundary (see Theorem~\ref{th:loc_loc})
while the corresponding submanifold of general nonlinear spurious solutions $M_{\mathrm{nonlinear}}
=\{f:\cube\to \R^d|\; f_\ast\nu=(f_0)_\ast\nu\}$ is infinite dimensional even when 
requiring $f(s)=f_0(s)$ close to the boundary of $\cube$.


\section{Results for volume preserving maps}\label{sec:volume}


Let us finally consider volume preserving transformations. For $\Omega \subset \R^d$, those are defined as the set of functions 
$
\mc{F}_{\mathrm{vol}}=\{f:\Omega\to \R^d \, | \; \det Df(x)=1\text{  for all $x\in \Omega$}\}.
$
Invertible volume preserving deformations have the property that they preserve the standard (Lebesgue)-measure $\lambda$
in the sense that $f_\ast \lambda_\Omega=\lambda_{f(\Omega)}$.
Recently it was proposed that volume preserving functions are a suitable function class for ICA. 
Here we show that those functions are not sufficiently rigid 
to allow identifiability of ICA in the unconditional case.
Note that Lemma~\ref{le:const_rotations} and Fact~\ref{fa:non_uniqueness}
already  show how to construct spurious solutions for the case that the base distribution is the uniform measure $\nu$. However, for an arbitrary distribution $\P$ this is slightly more difficult because we need to find maps $g$ that preserve $\P$, i.e., $g_\ast \P=\P$,
and are volume preserving, i.e., preserve the standard measure.
Nevertheless, we have the following theorem. 
\begin{theorem}\label{th:volume}
Let $p$ be a twice differentiable probability density with bounded gradient. Suppose that 
$x=f(s)$ where the distribution $\P$ of $s$ has density $p$ and $f$ is a diffeomorphism
with $\det Df(x)=1$ for $x\in \R^d$.
Then there is a family of functions $f_t:\R\times \R^d\to\R^d$ with $f_0=f$ and $f_t\neq f_0$ for $t\neq 0$ such that 
 $\det Df_t(x)=1$ and $(f_t)_\ast \P=f_\ast \P$.
\end{theorem}


The proof and an illustration are in Appendix~\ref{app:volume}.
It is based on the flows generated by  suitable explicit vector fields. As those flows can be concatenated we obtain a large family of spurious solutions. 
We think that the approach used here is a powerful technique to construct
counter-examples to identifiability in ICA.
Note that while $f$ is not identifiable it can be possible
to identify certain values of $f(s)$ when we the distribution of $s$ is known, e.g., volume preservation implies that the source value with the largest density is mapped to the point with the largest density in the observational distribution.
As local identifiability is weaker than identifiability we  get the following (informal) corollary.
\begin{corollary}
ICA in $\mc{F}_{\mathrm{vol}}$ is not identifiable.
\end{corollary}
Note that this is even true when we know the distribution of $s$.
The statement could be made rigorous by showing that if 
$(\mc{F}_{\mathrm{vol}}, \mc{P}_{\mathrm{vol}})$ is identifiable
with respect to $\mc{S}_{\mathrm{vol}}$ then $\mc{S}_\mathrm{vol}$
contains functions mixing coordinates $s_i,s_j$ for $i\neq j$ (i.e., there is 
$h\in \mc{S}_{\mathrm{vol}}$ such that $\partial_i\partial_j h\neq 0$).


\section{Relation to rigidity theory}\label{sec:rigidity}


Our results  rely on rigidity properties of certain function classes. Rigidity refers to the property that a local constraint on the derivative of a function implies
global restrictions on the shape of the function. 
For conformal maps the following rigidity result holds.
Let us clarify this through a well known example (a special case of Theorem~\ref{th:liouville} below).
\begin{theorem}\label{th:solid}
Let $f:\Omega\to \R^d$ with $\Omega\subset \R^d$ connected be a $C^1$ function
such that $Df(x)\in \mathrm{O}(n)$ for all $x\in \Omega$. Then $f(x)=Rx+b$ for some $R\in 
\mathrm{O}(d)$ and $b\in \R^d$.
\end{theorem}
This means that when the gradient of a function is a pointwise rotation, then the function is already a constant rotation.
Results of this type are of considerable interest in continuum mechanics. 
There it is natural to consider deformation of solids $f:\Omega\to \R^d$ 
which are locally constrained by the structure of the material. The condition
in Theorem~\ref{th:solid} corresponds to non-deformable solids while the 
condition $\det Df=1$ is used for incompressible fluids.

One important result in statistical mechanics is that
results like Theorem~\ref{th:solid} come with estimates in a neighbourhood of the function class, in 
the sense that if $Df(x)$ is close  to a rotation for all $x$, then it will be close to a 
constant rotation \cite{mueller, ciarlet1997mathematical}. Results of this type could be important for three reasons. First, in many real world applications it
might be more realistic to assume that $f$ is close to a certain function class $\mc{F}$
but not necessary contained in $\mc{F}$.

Secondly, it could rigorously justify the use of surrogate losses 
for the differential constraints (e.g., \cite{ima}
consider the loss $\sum_i \ln(|\partial_i f|)-\ln| \Det f|$)
\footnote{We denote the Euclidean norm by $|\cdot|$.} 
and, thirdly, results of this type would be needed 
for finite sample analysis.

We leave the further investigation of these matters to future work and return to 
 the rigidity result for conformal maps relevant for Theorem~\ref{th:conformal}.
 This is an extension of Theorem~\ref{th:solid} 
 where the condition $Df\in \mathrm{O}(d)$ is replaced by $Df\in \lambda \mathrm{O}(d)$ for some $\lambda\neq 0$. It turns
  out that this class is still very rigid in dimension $d\geq 3$ in the sense that there is only one additional solution.
The main input in this result is a theorem due to Liouville 
that shows that there are very few conformal maps for $n>2$.
\begin{theorem}[Liouville]\label{th:liouville}
For $n>2$, if $\Omega\subset \R^n$ is an open, connected set and
$f:\Omega\to \R^n$ is conformal, then


\begin{align}\label{eq:struct_conformal}
f(x)=b + \alpha \frac{A(x-a)}{|x-a|^\varepsilon}
\end{align} 


where $b, a \in \R^n$, $\alpha\in\R$, $A\in \mathrm{O}(n)$, and $\varepsilon\in \{0, 2\}$.
\end{theorem}
Originally this result was shown by Liouville \cite{liouville1850extension}, a modern treatment is \cite{geometric}.
In particular, this shows that conformal maps are (up to translations) rotations 
or rotations followed by an inversion.

To illustrate the strength of Theorem~\ref{th:liouville} we compare it to the setting of volume preserving maps which satisfy no similar rigidity property. 
Intuitively the different rigidity properties are already apparent from the connection to solids, which can merely be rotated and shifted, and fluids which can also be stirred leading to chaotic deformations.
Rigorously the different behaviours can be clarified by the observation that
conformal maps have a finite number of parameters and thus 
 a finite number of constraints (e.g., of the form $f(x)=y$) allows to identify them. In contrast volume preserving maps cannot be identified from finitely many constraints as the following proposition shows.
\begin{proposition}\label{prop:non_rigid}
For $d>2$ and $\{x_1,\ldots, x_n, y_1,\ldots, y_n\}\subset \R^d$, all pairwise different,
there is a volume preserving diffeomorphism $f:\R^d\to\R^d$, $f\in \mc{F}_{\mathrm{vol}}$ such that $f(x_i)=y_i$.
\end{proposition}
Let us emphasize that the different rigidity properties are not at all surprising when arguing based on  degrees of freedom or numbers of constraints. While volume preserving maps enforce only a single scalar constraint on the Jacobian $Df$ the condition for conformal maps gives
$n(n+1)/2 -1$ constraints on the Jacobian.

Let us finally comment on OCTs where
the picture is not as well understood. 
As discussed in Section~\ref{sec:ima} it is known \cite{darboux1896leccons, Cartan1925}
that OCTs constitute a rich, non-parametric class of functions and therefore OCTs are much more flexible than conformal maps. We illustrated this with example OCT constructions in Section~\ref{sec:ima} leveraging 2D conformal maps. 
Nevertheless, it is not known if and what rigidity properties can be derived for OCTs.
However, our results suggest that 
the additional measure preservation condition in the context of ICA gives 
enough rigidity to (almost) give identifiability of ICA.
In this sense OCTs might be a good function class for ICA as it is rich enough to allow complex representations of data while at the same time being sufficiently rigid to still provide a notion of identifiability whose strength remains to be determined.





\section{Discussion}\label{sec:discussion}


ICA is long known to be identifiable for linear maps, baring pathological cases, and highly non-identifiable for general nonlinear ones. Surprisingly, similar results for function classes of intermediate complexity remain scarce. 
In this work we address this question with several identifiability results for different function classes. 
Our first main result is that ICA is identifiable in the class of conformal maps (up to classical ambiguities). This considerably extends previous claims, limited to a specific 2D setting \cite{nonlinear_ica}, and ruling out several families of spurious solutions \cite{ima}. 
On the negative side we show that the ICA problem for volume preserving maps admits a large class of spurious solutions. 
Finally, we show that OCTs satisfy certain weaker notions of local identifiability. 

In our proofs, we draw connections to methods and techniques that, to the best of our knowledge, have not been used in the context of ICA before. We relate the identifiability problem in ICA to the rigidity of the considered function class $\mc{F}$ and use tools from the theory of partial differential equations. These techniques have been applied very successfully to the analysis of elastic solids \cite{ciarlet2021mathematical, ciarlet1997mathematical}
and 
we believe that there are many applications of these methods in the field of ICA. 

While the main focus of current research after the seminal work of \citet{ica_auxiliary} is on the auxiliary variable case, there are three reasons to consider unconditional ICA. Firstly, it is a fundamental research question that is, as illustrated by our results, deeply rooted in functional analysis. Secondly there is high application potential for completely unsupervised learning without any auxiliary variables, as the corresponding datasets do not require labelling or specific experimental settings. Thirdly, it is very likely that the techniques can be generalised to the auxiliary variable case.

Another important open problem is assessing the type of constraints on ground truth mechanisms, encoded by function classes, that are relevant for real world data. It is plausible that those mechanisms are typically much more regular than generic nonlinear functions. Recently, \citet{ima} suggested, based on arguments from the causality literature that $\mc{F}_\ima$ is a natural class for representation learning (and our results show it also has favourable theoretical properties), but this will require experimental confirmation on real world data.

Finally, a central question from a machine learning perspective is the ability to design learning algorithms that can train LVMs with identifiable function class constraints. Interestingly, \citet{ima} showed that OCT maps can be learnt using a closed from regularized likelihood loss, thereby providing, supported by our result, a full-fledged identifiable nonlinear ICA framework.

\bibliography{./arxiv.bib}
\newpage

\begin{appendix}
\part*{Appendix}
In the appendix we provide the proofs of our results and
we  discuss the relevant background. It is structured as follows.
We first introduce some mathematical background in Appendix~\ref{app:math} and
extend the function class definitions to Riemannian manifolds in Appendix~\ref{app:ext}. 
We discuss a general construction of spurious solutions in Appendix~\ref{app:spurious}.
Then we provide the proofs of our main results in
Sections~\ref{app:conformal} to \ref{app:volume}.
\section{Mathematical background}\label{app:math}
For the convenience of the reader we collect some mathematical definitions and notations.
All definitions can be found in standard textbooks.

\paragraph{Pushforward of measures.} For a measure $\mu$ on a (measurable) space $X$ and a measurable map, $f:X\to Y$ the pushforward measure $f_\ast \mu$ is defined by
$(f_\ast \mu)(A)=\mu(f^{-1}(A))$ for any measurable set $A\subset Y$. Here $f^{-1}(A)=\{x\in X\, | \;f(x)\in A\}$
denotes the preimage of $A$ under $f$. Sometimes the pushforward measure is denoted by $f^\#\mu$. 
Note that no further restrictions on $f$ are necessary, in particular $f$ does not need to be invertible. 
One important property that we will use frequently is the relation $(g\circ f)_\ast \mu=f_\ast( g_\ast \mu)$.

Note that if $s\sim \P$, i.e., $s$ has distribution $\P$ then $f(s)\sim f_\ast\P$. 
Indeed, this is obvious as $
P(f(s)\in A)=P(s\in f^{-1}(A))=\P(f^{-1}(A))$.
In the context of ICA it is convenient to mostly talk about distributions and push-forwards as we typically never observe pairs $(s, f(s))$.
For later usage we also recall the transformation formula for random variables.
If $f\in C^1(\R^n,\R^n)$ is an invertible diffeomorphism and 
$\Q=f_\ast \P$ where $\P$ and $\Q$ have density $p$ and $q$ respectively
then 
\begin{align}\label{eq:density_pushforward_2}
q(y) = p(f^{-1}(y)) |\Det D f^{-1}(y)|.
\end{align}
Note that the potentially more familiar version for random variables reads as follows.
Let
$X$ and $Y$ be random variables satisfying $Y=f(X)$, then their densities are related by
\begin{align}\label{eq:density_pushforward}
p_Y(y) = p_X(f^{-1}(y)) |\Det D f^{-1}(y)|.
\end{align}

\paragraph{Diffeomorphisms.} Let $U, V\subset\R^d$. A diffeomorphism 
from $U$ to $V$ is a bijective map $f:U\to V$
such that $f\in C^1(U,V)$ and $f^{-1}\in C^1(V, U)$.
Note that it is not sufficient to assume that $f$ is bijective and in $C^{1}(U,V)$, the classic counterexample being $f(x)=x^3$. 
A sufficient condition is that $Df(x)$ is an invertible matrix for all $x$.
Sometimes we loosely speak about diffeomorphisms $f:U\to \R^d$ which should be understood as $f$ being a diffeomorphism on its image $f(U)$.

\paragraph{Vector fields and flows}
Vector fields can be introduced nicely in the language of differential geometry. However, we think that for the purpose of this paper it is more appropriate to give a more down to earth discussion focused on $\R^d$. 
We refer to \cite{lee2003introduction} for a general introduction.
A vector field is a map $X:\R^d\to \R^d$. Our reason to consider vector fields is that they can be used to describe smooth transformations of $\R^d$. 
We will assume that $X$ is Lipschitz continuous. 
We define the flow of a vector field as a map $\Phi:\R\times \R^d\to \R^d$
such that
\begin{align}\label{eq:def_flow}
\Phi_0(x)=x,\qquad \partial_t \Phi_t(x)=X(\Phi_t(x)).
\end{align}
Let us remark concerning the notation that it is convenient to put the
$t$ argument below, i.e., we write $\Phi_t(x)=\Phi(t,x)$. Moreover, when applying differential operators $D$ they will by default only act on the spatial variable $x$, i.e., $D\Phi_t(x)$ denotes the derivative of the function $x\to \Phi_t(x)$ for a fixed $t$. 
Note that the solutions of differential equations can exhibit blow-up so $\Phi_t(x)$ might not be defined for all times $t$.
However, when we assume that $X$ is bounded the flow exists globally.

It can be shown 
that 
if $X$ is $k$ times differentiable then so is $\Phi_t$ and one can conclude
that then $\Phi_t$ is a diffeomorphism ($\Phi$ is bijective by the uniqueness of 
ordinary differential equations).
We will also consider time-dependent vector fields $X:(-T,T)\times \R^d\to \R^d$
where the flows can be defined similarly, replacing $X$ by $X_t$.

We are particularly interested in the action of flows on probability measures, i.e.,
we consider  the measures  $(\Phi_t)_\ast \P$ for some initial measure $\P$.
It can be shown that if the density of $\P$ is $p$ then the density 
$p_t$ of $(\Phi_t)_\ast \P$ satisfies the continuity equation
\begin{align}\label{eq:continuity}
\partial_t p_t + \Div(p_t X_t)= 0\quad \text{and}\quad p_0=p.
\end{align}
Here $\Div$ denotes the divergence which is defined by
$\Div f=\sum_i \partial_i f_i=\tr Df$.
One important consequence is that the flow of the vector field preserves the measure,
i.e., $(\Phi_t)_\ast \P$ iff $\Div(pX_t)=0$ for all $t$.
Moreover, the standard Lebesgue measure in $\R^d$ is preserved
if $\Div(X_t)=0$, i.e., divergence free vector fields generate volume preserving flows and vice versa.

\paragraph{Additional notation.}
We at some places use the notation $[n]=\{1,\ldots, n\}$. We also use the $\Onot$
notation. Recall that $f(x)=\Onot(g(x))$ as $x\to \infty$ means that there are constants $x_0$ and $C>0$ such that $f(x)\leq Cg(x)$ for $x\geq x_0$.
Recall that we introduced the notation $\cube=(0,1)^c$ in the main part and we denoted by $\nu$ the uniform measure on $\cube$.
We write iff as a shorthand for 'if and only if'.

{

\section{Function class definitions for general manifolds}\label{app:ext}

Here we extend the definition of the function classes to Riemannian manifolds. 
Riemannian manifolds are manifolds $M$ equipped with a metric $g$. For complete definitions we again refer to \cite{lee2003introduction}.
We denote the differential of a map $f:M\to N$ at $x\in M$
by $(\d f)_s:T_sM\to T_{f(s)}N$. As before, we use the notation $Df(s)\in \R^{d\times d'}$
for the usual derivative of a map $f:\R^d\to \R^{d'}$. 
The matrix $Df(s)$ is the representation of $(\d f)_s$ in the standard basis.
A smooth  map $f:M\to N$ between Riemannian manifolds $(M,g)$ and $(N,h)$ is conformal
if there is a function $\lambda:M\to \R_+$ such that
\begin{align}
   (f^\ast h)=\lambda g
\end{align}
where $f^\ast h$ denotes the pullback metric.
This means that for $X,Y\in T_xM$ 
\begin{align}
    h((\d f)_sX,(\d. f)_sY)=\lambda(s)g(X,Y).
\end{align}
Moreover, we observe that the adjoint $(df)_s^\ast$ satisfies by definition
\begin{align}
   g((\d f)_s^\ast (\d f)_s X, Y)=
   h((\d f)_sX,(\d f)_sY)=\lambda(s)g(X,Y).
\end{align}
We conclude that 
\begin{align}
    (\d f)_s^\ast (\d f)_s =\lambda(s)\cdot  \mathrm{Id}_{T_sM}.
\end{align}
In the case that $M=N=\R^d$ both equipped with the standard metric the definition
is seen to be equivalent to Definition~\ref{def:conformal} given in the main part.
When $N$ is a $d$-dimensional submanifold of $\R^{d'}$ with $d'>d$ 
then we obtain the pointwise condition 
\begin{align}
    Df(s)^\top Df(s)=\lambda(s)\mathrm{Id}_{d\times d},
\end{align}
i.e., $Df(s)$ has orthogonal columns with equal norm.
Note that the concatenation of conformal maps is conformal and the inverse of conformal diffeomorphisms is again conformal.

Next, we extend the definition of orthogonal coordinate transformations.
We remark that orthogonal coordinates are chart dependent so no coordinate free definition can be given. Thus we focus on the case where
the domain of the function is $\R^d$
which we equip with the standard metric
and we consider the standard orthogonal
coordinate vector fields which we denote by $e_i$.
Then we call a smooth  map $f:\R^d \to M$ for a Riemannian manifold $(M,g)$
an orthogonal coordinate transformation if for $i\neq j$ and all $s\in \R^d$
\begin{align}
    g((\d f)_s e_i, (\d f)_s e_j)=0.
\end{align}
Again, this condition can be equivalently written as
\begin{align}
  \langle e_i  (\d f)_s^\ast (\d f)_s e_j\rangle_{\R^d}=0 
\end{align}
for $i\neq j$.
If $M$ is a submanifold of $\R^{d'}$ with the metric induced from the standard Euclidean metric we obtain the natural generalization 
of Definition~\ref{def:ima}, i.e.,
\begin{align}
    Df(s)^\top Df(s)\in \diag(d).
\end{align}
An important remark is that orthogonal coordinates do not exist for all manifolds as there are obstructions. Manifolds with this property 
are called locally diagonalizable.
This is closely related to the representation capability of the function class.

}

\section{Measure preserving transformations and spurious solutions}\label{app:spurious}
In this section we review the construction of spurious solutions for the ICA problem. 
We assume that we consider ICA in the class $(\mc{F}, \mc{P})$
where $\mc{F}$ is a function class and $\mc{P}$ a class of probability measures.
We are interested in understanding the set of spurious solutions
for a pair $(f,\P)$, i.e., the set of pairs $(f', \P')\in \mc{F}\times\mc{P}$ such that
$f_\ast \P=f'_\ast \P'$. 
We now note that if we can construct $h$ such that
$h_\ast \P'=\P$ and set $f'=f\circ h$ we have  $f_\ast \P=f'_\ast \P'$.
Next we define the subset of right composable functions
\begin{align}
    \mc{F}^R=\{f\in \mc{F}|\; g\circ f\in \mc{F} \text{ for all $g\in \mc{F}$}\}
\end{align}
and similarly the subset of left-composable functions
\begin{align}
    \mc{F}^L=\{f\in \mc{F}|\; f\circ g\in \mc{F} \text{ for all $g\in \mc{F}$}\}.
\end{align}
We remark that if $\mc{F}$ is a group then obviously $\mc{F}^L=\mc{F}^R=\mc{F}$ 
and the problem reduces to finding measure preserving transformations in $\mc{F}$.
The classes $\mc{F}_{\mathrm{lin}}$, $\mc{F}_{\mathrm{nonlinear}}$, $\mc{F}_\conf$, and $\mc{F}_{\mathrm{volume}}$ are all groups. 
We will comment on  $\mc{F}_\ima$ below.

We note that if $h\in \mc{F}^R$ satisfies
$h_\ast \P'=\P$ then $f'=f\circ h \in \mc{F}$ and $f'_\ast \P'=f_\ast \P$
so this gives us a spurious solution. 
A specific case is given by $\P=\P'$ in which case we are looking for measure preserving transformations (MPTs) $h\in \mc{F}^R$.
Note that such $h$ for a certain $\P$ allows us to construct spurious solutions for all pairs $(f, \P)$ with arbitrary $f$.

This implies that if $(\mc{F},\mc{P})$ is identifiable with respect to $\mc{S}$
then any $h$ as above satisfies $h\in \mc{S}$ (if $\P'$ has full support, and otherwise there is a $h'\in \mc{S}$ such that $h$ and $h'$ agree in the support of $\P'$).

We consider an example. Let $\P$ be the distribution of the standard Gaussian and
$h=R$ for some $R\in \mathrm{O}(d)$. Then $h\in \mc{F}_{\mathrm{lin}}$
and $h_\ast \P=\P$ as the standard Gaussian is invariant under rotations.
However, such a linear measure preserving transformation does not exist
for other $\P\in \mc{M}_1(\R)^{\otimes d}$.
This is the reason that Gaussian distributions are excluded for linear ICA.

Next we observe (as we discussed in the main part) that  if the function class $\mc{F}$ is stable by right composition with arbitrary component-wise transformation we can turn any admissible latent distribution into another one. This means we can fix a reference measure (we will usually use $\nu$) and then  we can find (at least under suitable regularity assumptions) for $\P\in \mc{M}_1(\R)^{\otimes d}$
maps $h^{\P\to\nu},h^{\nu\to\P}\in \mc{F}^R$ such that
$h^{\P\to\nu}_\ast \P=\nu$, $h^{\nu\to\P}_\ast \nu=\P$.
If we can find an MPT $h\in \mc{F}^R$ such that $h_\ast \nu=\nu$
mixing the coordinates we can then find spurious solutions for any pair $(f, \P)$
by considering $f'=f\circ h^{\nu\to\P}\circ h\circ h^{\P\to\nu}$ because then $f'_\ast\P=f_\ast\P$ and
$f'\in \mc{F}$ by definition of $\mc{F}^R$.

For  the class $\mc{F}_{\mathrm{nonlinear}}$ such MPTs exist, we gave one construction in Lemma~\ref{le:const_rotations}. 
In Appendix~\ref{app:ima} we will sketch the proof of this result and discuss another construction based on divergence free vector fields.
For the class $\mc{F}_{\mathrm{volume}}$ we cannot arbitrarily transform the input distribution. However, we can still find coordinate mixing MPTs for every $\P$
(with smooth density). This will be proved in Appendix~\ref{app:volume}.
These constructions rule out any (meaningful) identifiability result for $\mc{F}_{\mathrm{nonlinear}}$ and $\mc{F}_{\mathrm{volume}}$.

There is also a slightly different approach to construct spurious solutions for a pair $(f, \P)$. For any  
MPT $h\in \mc{F}^L$ such that $h_\ast f_\ast\P=f_\ast\P$, i.e., transformations that preserve the observational distribution $\Q=f_\ast \P$, the function $h\circ f$ defines a spurious solution. Note that this gives spurious solutions for all pairs $(f,\P)$ such that $f_\ast \P$ follows the fixed distribution $\Q$.
This approach will allow us to construct spurious solutions in $\mc{F}_\ima$ for
some fine-tuned pairs $(f,\P)$ with $f\in \mc{F}_\ima$ (see the proof of Proposition~\ref{prop:non-uniqueness} in Appendix~\ref{app:ima}).
Note that because of their fine-tuning, such spurious solutions can arguably be considered pathological cases instead of key non-identifiability issues. This is inline with how such solutions a considered in the case of linear ICA.


Let us finally have a closer look at $\mc{F}_\ima$. 
It is quite straightforward to see that
\begin{align}
    \mc{F}_\ima^L = \mc{F}_\conf.
\end{align}
In particular, all rotations are contained in $\mc{F}_\conf^L$.
More interestingly, we have
\begin{align}
\begin{split}\label{eq:FR_ima}
    \mc{F}_\ima^R &=\{f\in \mc{F}_\ima |\, Df(x)=P(x)\Lambda(x) \;\text{for some $P(x)\in \perm(d)$, $\Lambda(x)\in \diag(d)$}\}
    \\
    &=
    \{ f\in \mc{F}_\ima|\, f=P\circ h\, \text{where $P\in \perm_\pm(d)$ and
$h:\R^d\to\R^d$ with }
\\
&\qquad\quad \text{ $h(x)=(h_1(x_1),\ldots, h_d(x_d))^\top$
for some $h_i\in C^1(\R,\R)$ with $h_i'>0$}\}.
    \end{split}
\end{align}
The first step can be seen using the chain rule and the definition of $\mc{F}_\ima$.
The second step follows from the fact that the permutation $P$ is necessarily constant for such a function. This again recovers the fact that OCTs can only be identified up to permutations and coordinate-wise reparametrisations. However, we also conclude that 
all $h\in \mc{F}_\ima^R$ 
act coordinate-wise (up to a permutation), i.e., do not prevent BSS. This shows that for $\mc{F}_\ima$
no completely generic spurious solution based on a single MPT mixing the coordinates exists. This is different from $\mc{F}_{\mathrm{nonlinear}}$. We emphasize that this does not rule out the existence of a fine-tuned spurious solution for every (or most) pairs $(f,\P)$ because then we only need to find $h$ such that $h_\ast \P=\P$ and $f\circ h\in \mc{F}_\ima$.
The relation $f\circ h\in \mc{F}_\ima$ of course holds for $h\in \mc{F}_\ima^R$ but there will, in general, be many more $h$ for a fixed $f$.

\section{Proof of Theorem~\ref{th:linear}}\label{app:linear}
We here, for completeness, give a proof of Theorem~\ref{th:linear}.
While this result is well known we 
think that it makes sense to include a condensed proof because it
contains many of the key steps of the more involved proof for conformal maps
in the next section
and it is not as well known as the proof based on Darmois-Skitovich Theorem
which does not generalise to nonlinear functions.

\begin{proof}[Proof of Theorem~\ref{th:linear}]
We assume that $x\indistribution As\indistribution A's'$. We first assume that the densities $p:\R^n\to \R$ and $q:\R^n\to \R$ 
of $s$ and $s'$ 
are $C^2$ functions and $p(x)>0$ for all $x\in \R^d$. We denote their distributions by $\mathbb{P}$ and $\mathbb{Q}$.
 By independence of the components we can write
$p(x)=\prod_i p_i(x_i)$, $q(x)=\prod_i q_i(x_i)$ for some $C^2$ functions
$p_i$ and $q_i$.
By assumption we conclude that $((A')^{-1}A)_\ast \mathbb{P}=\mathbb{Q}$.
We denote $B = ((A')^{-1}A)^{-1}=A^{-1}A'$ so that $B^{-1}_\ast\P= \Q$ and the transformation formula for densities implies that 
\begin{align}\label{eq:density_linear}
 q(y) = p(By) \, |\Det B|\quad \Rightarrow
 \quad \sum_k \ln(q_k(y_k)) = \sum_k \ln(p_k((By)_k) + \ln|\Det B|.
\end{align}
The main idea of the proof is to use the observation that for $i\neq j$
and all $y$ such that $q(y)\neq 0$
\begin{align}\label{eq:key_ident}
\partial_i\partial_j \ln(q(y)) = 
\partial_i\partial_j \sum_k \ln(q_k(y_k))=0
\end{align}
i.e., mixed second derivatives of the log density vanish. We plug  \eqref{eq:density_linear} into this equation and get
\begin{align}
\partial_i \ln(q(y)) &= \sum_k B_{ki} \ln(p_k)'((By)_k),
\\
\partial_i\partial_j \ln q(y) &= \sum_k B_{kj}B_{ki}\ln(p_k)''((By)_k).
\end{align}
We now denote $x=By$.
Then we get (using that $B$ is invertible) that for all $x$ such that $p(x)\neq 0$
\begin{align}\label{eq:cond_linear_new}
0 = \sum_k B_{kj}B_{ki}\ln(p_k)''(x_k).
\end{align}
Varying one $x_k$ individually we conclude that each summand is constant which implies that either $B_{kj}B_{ki}=0$ for all $i\neq j$ or $\ln(p_k)''(x_k)$
is constant.
It is straightforward to see that the only probability distribution 
with $\ln(p)''=\kappa$ for some constant $\kappa$ are Gaussian distributions.
Indeed, note that $\ln(p)''=\kappa$ for some constant implies that
$p(x)=\exp(\kappa x^2/2+c_1x+c_2)$. If $p$ is the density of a probability distribution 
and in particular integrable we see that this implies $\kappa<0$
and $p$ is a Gaussian density with $\kappa=-\sigma^{-2}$.
Note that by assumption at most one component of $\P$ is Gaussian, w.l.o.g., $k=1$.

As  $p_k$ is not a Gaussian for $k>1$ and thus $\ln(p_k)''$ not constant we conclude
\begin{align}\label{eq:Bk}
B_{kj}B_{ki}=0
\end{align}
 for $i\neq j$ and all $k>1$.
 Plugging this into \eqref{eq:cond_linear_new} we obtain 
 \begin{align}
    0= B_{1j}B_{1i}\ln(p_1)''(x_1).
 \end{align}
 Note that if $p_1$ is a Gaussian density then $\ln(p_1)''=-\sigma^{-2}\neq 0$
 so we conclude that in any case $B_{1i}B_{1j}=0$ for $i\neq j$, i.e., 
 \eqref{eq:Bk} holds as well for $k=1$.
 
 In other words, at most one entry of every row of $B$ is non-zero. This implies that
 $B=P\Lambda$ for some $P\in \mathrm{Perm}(d)$ and $\Lambda\in \mathrm{Diag}(d)$.
 This was to be shown.
 
 Let us clarify what happens when there is more than one Gaussian component. 
 In this case there  might be multiple constant non-zero terms in \eqref{eq:cond_linear_new}
whose 
contributions can cancel and we cannot conclude that \eqref{eq:Bk} holds
for all $k$. This recovers the well known non-uniqueness for Gaussian variables.

It remains to extend the result to distributions whose density is not twice differentiable.
By standardizing $s$ and $s'$
we can assume that $B\in \mathrm{O}(n)$. Indeed, when the covariances of $s$ and $s'$ satisfy $\Sigma_s=\Sigma_{s'}=\mathrm{Id}_d$ then $s=Bs'$ implies $B^\top B=B^T\Sigma_{s'}B=
\Sigma_s=\mathrm{Id}$.
Then 
$s'\stackrel{\mc{D}}{=} Bs$ implies that for an independent standard normal
\begin{align}
B(s+N) \stackrel{\mc{D}}{=} s'+BN\stackrel{\mc{D}}{=}s'+N
\end{align}
where we used that standard normal variables are invariant under orthogonal maps.
Note that $s+N\in \mc{P}_\lin$. Indeed, $(s_i, N_i) \indep (s_j, N_j)$ implies
$(s_i+N_i)\indep (s_j+N_j)$ for $i\neq j$. If $s_i+N_i$ is Gaussian then
$s_i$ is Gaussian (consider, e.g., the Fourier transform) so $s+N$ also has at most one Gaussian component.
The density of $s+N$ is smooth and pointwise positive, so we can apply the reasoning above to $s+N$ and $s'+N$ and conclude
$B\in \mc{S}_\lin$.
 \end{proof}

\section{Proofs for the results on conformal maps}\label{app:conformal}

In this section we give the proofs for Section~\ref{sec:conformal}.
First, we consider $d>2$ and then the special case $d=2$.
\subsection{Proof of Theorem~\ref{th:conformal}}
The proof of Theorem~\ref{th:conformal} uses similar ideas as the proof for Theorem~\ref{th:linear}, however, the calculations
are more involved.
The key ingredient is the classification of all conformal maps 
in Theorem~\ref{th:liouville}. From there we see that we already dealt with the linear case in Theorem~\ref{th:linear}, so
it is sufficient to focus on the case of nonlinear transformations.
Recall that the Moebius transformations introduced in \eqref{eq:struct_conformal}
are given by
    \begin{align}\label{eq:struct_conformal2}
f(x)=b + \alpha \frac{A(x-a)}{|x-a|^\varepsilon}
\end{align} 
where $b, a \in \R^d$, $\alpha\in\R$, $A\in \mathrm{O}(d)$, and $\varepsilon\in \{0, 2\}$.
In particular, the Theorem will be an easy consequence of the following lemma.

\begin{lemma}\label{le:inversion}
Suppose $g:\R^d\to\R^d$ is a nonlinear Moebius transformation, i.e., a map as in \eqref{eq:struct_conformal2} with $\varepsilon=2$. Let $s$, $s'$ be random variables whose
distributions are in $\mc{P}_\conf$. Then $s\stackrel{\mc{D}}{\neq}  g(s')$.
\end{lemma}
Let us quickly show how it implies Theorem~\ref{th:conformal}
before we prove this lemma.
\begin{proof}[Proof of Theorem~\ref{th:conformal}]
We use the same notation as in the proof of Theorem~\ref{th:linear}. 
We denote the distribution of $s$ and $s'$ by $\mathbb{P}$ and $\mathbb{Q}$
and we assume  $x\indistribution f(s)\indistribution f'(s')$
with $f, f':\R^d\to M$ conformal (see Appendix~\ref{app:ext} for the definition). 
This implies that $(f'^{-1}f)_\ast \mathbb{P}=\mathbb{Q}$.
We denote $g = ((f')^{-1}f)^{-1}=f^{-1}f'$ so that $\P=g_\ast \Q$.
{
Now we make the simple but important observation that
$g$ as a concatenation of a conformal map and the inverse of a conformal map is 
a conformal map from $\R^d$ to $\R^d$.}
Thus, we  can apply Liouville's Theorem to $g$ (recalled in Theorem~\ref{th:liouville})
which implies that 
\begin{align}\label{eq:struct_g}
g(y) = b + \frac{\alpha A(y-a)}{|y-a|^\varepsilon}
\end{align}
where $b,a\in \R^d$, $\alpha\in \R$, $A\in \mathrm{O}(d)$, $\varepsilon\in \{0,2\}$.
Using Lemma~\ref{le:inversion} we conclude that $\varepsilon=0$ and $g$ is linear.
Then we can apply Theorem~\ref{th:linear} (using that $\mc{P}_\conf\subset \mc{P}_\lin$) and conclude that $\alpha A=P\Lambda$ for a permutation matrix and a diagonal matrix $\Lambda$. Since $g$ is conformal we have that $A$ is orthogonal
and all eigenvalues have absolute value 1 which implies that $\Lambda_{ii}=\pm \alpha$
for $1\leq i\leq d$ and thus $A\in M_{\mathrm{perm},\pm}(\R^{d\times d})$.
\end{proof}

To prove Lemma~\ref{le:inversion} we need one technical result 
that shows that local properties of the density $p_i$ (i.e., properties that hold 
for $x_i\in O_i$ for some  non-empty open sets $O_i$) in fact hold
for all $x_i\neq 0$.
This will be based on the nonlinearity of the map $g$
combined with the factorisation $p(x)=\prod_i p_i(x_i)$ of the densities.
An illustration can be found below in Figure~\ref{fig:conformal}.

\begin{lemma}\label{le:geometry}
Let $O=O_1\times \ldots\times O_d\subset \R^d$ 
and $U=U_1\times \ldots \times U_d\subset \R^d$ where $O_i$ and $U_i$ are non-empty open sets. Let $g(x)= Ax/|x|^2$ for an orthogonal matrix $A$. Assume that $g(O)=U$.
Then $U_i$ is either $(0,\infty)$, $(-\infty,0)$
$(-\infty,0)\cup(0,\infty)$, or $(-\infty,\infty)$.
Moreover, if the $i$-th row of $A$ is not equal to a (signed) standard basis vector then $U_i=\R$.
\end{lemma}
Informally the result follows from the fact that coordinate planes $\{x_i=c\}$
are mapped to spheres by $g$ except for $c=0$. We assume that the boundary of $O$ and $U$ is the union of subsets of hyperplanes. 
 However, $g(O)=U$ then implies
that the boundaries of $O$ and $U$ are mapped to each other. Since the boundaries
of both sets are the union of subsets of hyperplanes that are mapped to hyperplanes we conclude their boundaries  must be a subset of the coordinate axes hyperplanes $H_i=\{x:\, x_i=0\}$ (because they would be mapped to spherical caps by $g$).
For completeness, we give a careful proof below.
Let us highlight that this lemma essentially rules out counterexamples with finite support because we can apply the lemma to the set $\{x\in \R^d: p(x)>0\}$. This is in contrast
to the 2-dimensional case where conformal maps between any two rectangles exist
making the proof of the corresponding statement below more difficult.

We now prove Lemma~\ref{le:inversion}.
\begin{proof}[Proof of Lemma~\ref{le:inversion}]
The proof is a bit lengthy, so we first give an informal overview of the main steps.
As before, we denote the distribution of $s$ and $s'$ by $\mathbb{P}$ and $\mathbb{Q}$. 
We argue by contradiction, so we assume that
 $ \P =g_\ast \mathbb{Q}$.
By assumption $g(x)=b+\alpha A(x-a)/|x-a|^2$.
We now proceed in several steps that 
constrain the structure of $g$. Let us briefly describe the steps of the proof.

In Step 1 and 2 we eliminate the trivial symmetries of the measure and show 
that the mild local regularity assumption on the measures imply global regularity.

Then, in Steps 3 and 4 we derive in \eqref{eq:main_condition} a condition similar to \eqref{eq:cond_linear_new} but more involved.

To exploit this condition we look in Steps 5 and 6 at certain limiting
regimes where the terms become much simpler and almost reduce to the condition of the linear case. 
This allows us to conclude that $A$ is a permutation matrix in 
Step 7. 

Using that $A$ is a permutation matrix in \eqref{eq:main_condition} we get in Step 8 a much simpler relation
that almost factorizes. This allows us to derive a simple differential equation
in Step 9 which restricts the potential densities to a simple parametric form. 
 This allows us to conclude.

\paragraph{Step 1: Elimination of trivial symmetries.}
First we show that we can assume $a=b=0$ and $\alpha=1$.
We denote the shifts on $\R^d$ by $T_a(x)=x+a$  and the dilations
$D_\alpha(x)=\alpha x$. Then we can rewrite
$g = T_b \circ D_\alpha\circ g_0 \circ T_{-a}$ with $g_0(x)= Ax/|x|^2$ and therefore 
$(g_0)_\ast (T_{-a})_\ast \mathbb{Q}=(D_{\alpha^{-1}}\circ T_{-b})_\ast \mathbb{P}$. 
Since shifts and dilations preserve the class $\mc{P}_\conf$
it is sufficient to show the result for $g=g_0$.
To simplify the notation we drop the $0$ in the following and just assume $g(x)=Ax/|x|^2$.



\paragraph{Step 2: Support and smoothness of distributions.}
The goal of this step is to show that under the assumption of Lemma 2
the density of $\P$ and similarly of $\Q$ is positive and $C^2$ away from the coordinate hyperplanes $\{x_i=0\}$ for a union of quadrants, while it vanishes on the remaining quadrants. This will be a consequence of Lemma~\ref{le:geometry}.
Let $U_i=\mathrm{Int}(\mathrm{supp}\P_i)$. By assumption, $\P\in \mc{P}_\conf$ which entails $U_i\neq \emptyset $.
Then $U=U_1\times \ldots\times U_n=\mathrm{Int}(\mathrm{supp}\, \P)$.
Define $O_i$ similarly	for $\Q$. The relation $g_\ast \Q=\P$ implies
$g(O)=U$. Applying Lemma~\ref{le:geometry} we conclude that $U$ is the union of
quadrants. 

The same argument will imply that $p$ is actually $C^2$ away from the coordinate planes.
We  consider the interior of the set of points where $\P$ has a twice differentiable and positive density and call it $U'$. For $x\in U'$ there is a density $p$ in a neighbourhood of $x$ and it factorizes by the independence assumption. The relation 
\begin{align}
    \partial_i^2 p(x) /p(x) = \partial_i^2p_i(x_i)/p_i(x_i)
\end{align}
implies that
then $p_i$ is twice differentiable at $x_i$. Vice-versa, if all $p_i$ are twice differentiable at $x_i$ then $p$ is twice differentiable at $x=(x_1,\ldots,x_d)$.
This implies that $U'=U_1'\times\ldots\times U_d'$ for some open sets $U_i'$.
By definition of $\mc{P}_\conf$ the sets $U_i'$ are non-empty.
Similarly, we define $O'$. The relation \eqref{eq:density_pushforward}
for the densities $p$ and $q$
implies, together with the smoothness of $g$, that 
$g(O')=U'$. Then we apply again Lemma~\ref{le:geometry} to conclude that 
$p$ and $q$ are $C^2$ functions away from the hyperplanes $\{x_i=0\}$ 
and if the density is non-zero at a point in a quadrant then it is positive in the complete interior of the quadrant.

Finally, $U_i'=\R$ if
the $i$-th row of $A$ has more than one non-zero entry (again by Lemma~\ref{le:geometry}). By definition this means that
$p_i\in C^2(\R)$ and $p_i(x)>0$ for such $i$ and all $x$.

Let us emphasize here, that we already finished the proof of Lemma~\ref{le:inversion} for probability distributions with bounded support. 
This is more difficult in dimension 2 because there are two-dimensional conformal functions mapping rectangles to rectangles. We will  
consider this in the proof of Theorem~\ref{th:conf2} below.

A major step in the proof is to show that $A$ is a permutation matrix under the assumptions of the lemma. 
We define the index set $I\subset [d]$ as the set of all indices $i$ such
that the $i$-th row of the matrix $A$ has only one non-zero entry.
Our goal is to show that $I=[d]$. We have shown so far that $p_k$ is positive
and twice differentiable if $k\notin I$.

%
%

The proof in the linear case relied on the relation \eqref{eq:cond_linear_new}.
We now derive a similar relation for non-linear Moebius transformations.
\paragraph{Step 3: Derivative formulas.}
For future reference we note (using $A\in \mathrm{O}(n)$) that
(for $i\neq j$)
\begin{align}\label{eq:dg}
(Dg(y))_{kj}&=\partial_j g_k(y) =\left( \frac{A}{|y|^2} - 2\frac{Ay\otimes y}{|y|^4}\right)_{kj}
= \frac{A_{kj}}{|y|^2} - 2\frac{(Ay)_k y_j}{|y|^4},
\\ \label{eq:ddg}
(\partial_i\partial_j g_k)(y)&= -2\frac{A_{kj}y_i+A_{ki}y_j}{|y|^4}+\frac{8(Ay)_ky_iy_j}{|y|^6}
\\ \label{eq:det_dg}
\Det(Dg(y))&=
\Det \frac{A}{|y|^2} \Det\left(\mathrm{Id}-2\frac{y\otimes y}{|y|^2}\right)
=-|y|^{-2d}.
\end{align}

 \paragraph{Step 4: Derivation of a condition for the densities} 
 We apply the same reasoning as 
in the proof of Theorem~\ref{th:linear} to derive partial differential equations
for the density $p$.
The condition $s=g(s')$, or equivalently $g^{-1}(s)=s'$
and the density relation \eqref{eq:density_pushforward} imply 
\begin{align}
q(y)=p(g(y))|\Det \nabla g(y)|.
\end{align}
This implies
 \begin{align}
 q(y) &= p\left( g(y)\right) |\Det \nabla (Ay|y|^{-2} )|=p\left(g(y)\right) \, |y|^{-2d}
\\
 \Rightarrow \sum_k \ln(q_k(y_k))
 &= \sum_k \ln(p_k(g_k(y))) -2d \ln(|y|).
 \end{align}
 We calculate for $i\neq j$
 \begin{align}
 \begin{split}
 \partial_i \ln(q(y)) &= \sum_k (\partial_i g_k)(y) (\ln p_k)'(g(y))
 -2d \partial_i\ln(|y|),
 \\
0= \partial_j\partial_i \ln(q(y)) &=
-2d \partial_j\partial_i \ln(|y|)
+\sum_k (\partial_i\partial_j g_k)(y) (\ln p_k)'(g_k(y))
\\
&+\sum_{k}(\partial_ig_k)(y) (\partial_j g_k)(y)(\ln p_k)''(g_k(y)).
 \end{split}
 \end{align}
 Evaluating the derivatives using \eqref{eq:dg} and \eqref{eq:ddg} 
 we get
 \begin{align}
 \begin{split}
 0 
&=  4d\frac{y_iy_j}{|y|^4} 
 \\
 &\quad +
 \sum_k 
 \left(\frac{8(Ay)_ky_iy_j}{|y|^6}-2\frac{A_{kj}y_i+A_{ki}y_j}{|y|^4}
 \right) \ln(p_k)'(g_k(y))
\\
&\quad  +
\sum_{k}
\left(\frac{A_{kj}}{|y|^2} - 2\frac{(Ay)_k y_j}{|y|^4}\right)
\left(\frac{A_{ki}}{|y|^2} - 2\frac{(Ay)_k y_i}{|y|^4}\right) \ln(p_k)''(g_k(y)).
 \end{split}
 \end{align}
Finally we express the variable $y$ through $x=g(y)=Ay|y|^{-2}$. Note that then $|y|=|x|^{-1}$ and 
$y=A^{-1}x|x|^{-2}$. Plugging this in the last equation we get
\begin{align}
 \begin{split}\label{eq:main_condition}
 0 
&=  4d(A^{-1}x)_i(A^{-1}x)_j
 \\
 &\quad +
 \sum_k \left(8x_k(A^{-1}x)_i(A^{-1}x)_j-2|x|^2\left(A_{kj}(A^{-1}x)_i+A_{ki}(A^{-1}x)_j\right)
 \right) \ln(p_k)'
\\
&\quad  +
\sum_{k}
\left(A_{kj}|x|^2 - 2x_k (A^{-1}x)_j\right)
\left(A_{ki}|x|^2 - 2x_k (A^{-1}x)_i\right) \ln(p_k)''
\\
&=4d A_{mi} x_m A_{lj}x_l
 \\
 &\quad +
 \sum_k \left(8x_kA_{mi} x_m A_{lj}x_l -2|x|^2\left(A_{kj}A_{mi} x_m+A_{ki}A_{lj}x_l\right)
 \right) \ln(p_k)'
\\
&\quad  +
\sum_{k}
\left(A_{kj}|x|^2 - 2x_k A_{lj} x_l\right)
\left(A_{ki}|x|^2 - 2x_k A_{mi} x_m\right) \ln(p_k)''
 \end{split}
\end{align} 
where we used $A^{-1}=A^\top$ as $A$ is orthogonal and we used Einstein summation convention to sum over indices that appear twice (we kept the sum over $k$ for better readability).
Note that this expression is not homogeneous in $x$. 

\paragraph{Step 5: Simplifications as $x_r\to\infty$.}
We fix an index $1\leq r\leq d$. 
The strategy is now to send $x_r\to \infty$ while keeping the other coordinates bounded. We can assume by reflecting coordinates 
that the quadrant $\{x_i>0, \forall i\}$ is contained in the support of $\P$ and 
$p$ has a positive $C^2$ density there.
We can then rewrite \eqref{eq:main_condition} 
\begin{align}
\begin{split}
0 &=\Onot (x_r^2)  +(8x_r^3A_{ri}A_{rj}-4x_r^3A_{rj}A_{ri})\ln(p_r)'+\Onot(x_r^2\ln(p_r)')+
\Onot(x_r^3)\\
&\quad+(A_{rj}x_r^2-2x_rA_{rj}x_r)(A_{ri}x_r^2-2x_rA_{ri}x_r)\ln(p_r)''
+\Onot(x_r^3\ln(p_r)'')
\\
&\quad
+\sum_{k\neq r} A_{kj}A_{ki}x_r^4\ln(p_k)''+\Onot(x_r^3)
\\
&=
x_r^4 \sum_k A_{kj}A_{ki}\ln(p_k)''+4x_r^3 \ln(p_r)' +\Onot\left(x_r^3(1+\ln(p_r)'')+x_r^2\ln(p_r)'\right).
\end{split}
\end{align}
We conclude that
\begin{align}\label{eq:cond_x_r}
0 = \sum_k A_{kj}A_{ki}\ln(p_k)''+\frac{4A_{ri}A_{rj} \ln(p_r)'}{x_r} +\Onot\left(\frac{1+\ln(p_r)''}{x_r}+\frac{\ln(p_r)'}{x_r^2}\right).
\end{align}
By varying $x_k\neq x_r$ this almost implies that $A_{kj}A_{ki}\ln(p_k)''=c_k$ for some constant
whenever $p_k>0$ and twice differentiable.  However, for this we need to show that
the terms hidden in  $\Onot(\cdot)$ are really negligible, i.e.\
$\ln(p_r)''$ and $\ln(p_r)'/x_r$ are bounded as $x_r\to\infty$ so that 
the remainder becomes $o(1)$ which we will establish. 

Note that if such a relation could be derived we could conclude, similarly to the linear case, that $A$ is a permutation matrix.

\paragraph{Step 6: Boundedness of $\ln(p_r)''$ and $\ln(p_r)'/x_r$.}
Recall that $I$ is the set of indices such that the $i$-th row of $A$ has only one non-zero entry. Let $r\notin I$ and pick $j, i$ such that $A_{rj}A_{ri}\neq 0$.
Fix all coordinates $x_k$ except $x_r$ so that $p_k$ is positive and twice differentiable at $x_k$.
Then we can express \eqref{eq:cond_x_r} as
\begin{align}\label{eq:diff_ineq}
\ln(p_r)'' + \frac{4 \ln(p_r)'}{x_r}=R(x_r)
\end{align}
where the remainder term  $R$ contains the remaining terms. The expression 
$R$ of course depends on the other coordinates $x_k$ for $k\neq r$ but since they are considered fixed here we can view $R$ as a function of $x_r$ alone.

Equation \eqref{eq:cond_x_r}
then implies that there is $M>0$ sufficiently large such that
for $x_r>M$ the remainder term
$R(x)$ 
satisfies for some constant $c>0$.
\begin{align}\label{eq:remainder}
|R(x)|\leq\frac12 |\ln(p_r)''| +\frac{|\ln(p_r)'|}{x_r}
+c.
\end{align}
Here the last constant term bounds the $x_r^{-1}$ contribution.
Suppose $\ln(p_r)'\leq 0$. Then we can bound
\begin{align}
\begin{split}
0&=\ln(p_r)''(x_r) + \frac{4 \ln(p_r)'(x_r)}{x_r}-R(x_r)
\\
&\leq \ln(p_r)''(x_r)+\frac12|\ln(p_r)''(x_r)|  + \frac{4 \ln(p_r)'(x_r)}{x_r} 
+ \frac{|\ln(p_r)'(x_r)|}{x_r}
+c\
\\
&\leq \ln(p_r)''(x_r)+\frac12|\ln(p_r)''(x_r)| +c 
\end{split}
\end{align}
We find that $\ln(p_r)''\geq -2c$ (for $\ln(p_r)''>0$ this is clear, and otherwise we can absorb the absolute value part). We conclude by integration that for all $x_r>M$
(note that the bound is trivially true if $\ln(p_r)'>0$)
\begin{align}
\ln(p_r)'(x_r)\geq \min(0, \ln(p_r)'(M))-2c(x_r-M).
\end{align}
Similarly we can bound for $\ln(p_r)'(x_r)\geq 0$
\begin{align}
\begin{split}
0&=\ln(p_r)''(x_r) + \frac{4 \ln(p_r)'(x_r)}{x_r}-R(x_r)
\\
&\geq \ln(p_r)''(x_r)-\frac12|\ln(p_r)''(x_r)|  + \frac{4 \ln(p_r)'(x_r)}{x_r} 
- \frac{|\ln(p_r)'(x_r)|}{x_r}
-c
\\
&\geq \ln(p_r)''(x_r)-\frac12|\ln(p_r)''(x_r)| -c 
\end{split}
\end{align}
implying $\ln(p_r)''(x_r)\leq 2c$ for $x_r\geq M$ such that $\ln(p_r)'(x_r)>0$.
We obtain 
\begin{align}
   \ln(p_r)'(x_r)\leq \max(0, \ln(p_r)'(M))+2c(x_r-M). 
\end{align}
Together the last two steps imply that $|\ln(p_r)'(x_r)|\leq C+Cx_r$ for some
$C>0$ and $x_r>M$. Going back to \eqref{eq:diff_ineq} we conclude that there is $C>0$ such
that $|\ln(p_r)''(x_r)|\leq C$ for $x_r>M$.
We conclude that for $r\notin I$ and all $i\neq j$
\begin{align}\label{eq:limit_cond_simplified}
0 = \sum_k A_{kj}A_{ki}\ln(p_k)''+\frac{4A_{ri}A_{rj} \ln(p_r)'}{x_r} +O\left(x_r^{-1}\right).
\end{align}

\paragraph{Step 7: $A$ is a permutation matrix.}
If $I=[n]$ we are done. So there is $r\notin I$ and using 
\eqref{eq:limit_cond_simplified} we conclude by varying $x_k\neq x_r$ that 
\begin{align}\label{eq:cond_constant}
A_{kj}A_{ki}\ln(p_k)''=c
\end{align}
for some constant (depending on $k$, $i$, and $j$).
Note that if we assumed that at most one $p_k$ is a Gaussian density we could conclude as in the linear case. However, this assumption is not necessary, as we will now show.

By definition, $A_{kj}A_{ki}\ln(p_k)''=0$ for $k\in I$ because there is only one non-zero entry in row $k$ of $A$.
We have seen in Step 2 that for $k\notin I$ the density $p_k\in C^2(\R)$  and is positive.
By assumption, we can find $i\neq j$ such that
$A_{kj}A_{ki}\neq 0$ for $k\notin I$. The
 relation \eqref{eq:cond_constant}
then implies for $k\notin I$ that $\ln(p_k)''(x_k)=\beta_k$ for some constant $\beta_k<0$ ($p_k$ is a probability density) and all $x_k\in \R$.
Then $\ln(p_k)'(x_k)=\beta_k x_k+\gamma_k$ for some constant $\gamma_k$.
With $x_r\to \infty$ we conclude from \eqref{eq:limit_cond_simplified}
that for $r\notin I$ 
\begin{align}\label{eq:tech1}
0 = \sum_{k\notin I} A_{kj}A_{ki}\beta_k+ 4A_{ri}A_{rj} \beta_r .
\end{align}
 Summing this over $r\notin I$ we get 
 \begin{align}
 \begin{split}\label{eq:tech2}
0 = \sum_{r\notin I}\left(\sum_{k\notin I} A_{kj}A_{ki}\beta_k+ 4A_{ri}A_{rj} \beta_r\right)
=
(d-|I|+4)\sum_{k\notin I} A_{kj}A_{ki}\beta_k.
 \end{split}
 \end{align}
Dividing  \eqref{eq:tech2} by $d-|I|+4$ and subtracting it  from \eqref{eq:tech1} we conclude that
\begin{align}
A_{ri}A_{rj} \beta_r=0
\end{align}
for all $r\notin I$ and all $i\neq j$. Since $\beta_r$ is non-zero this implies
that $A_{ri}A_{rj} =0$ for $i\neq j$ and thus $r\in I$, a contradiction.
This establishes $I=[d]$ and thus $A$ is a signed permutation matrix.
By permuting and reflecting the coordinates of $\P$ we can assume $A=\mathrm{Id}$ in the following.

\paragraph{Step 8: Simplifications in \eqref{eq:main_condition} for $A=\mathrm{Id}$.}
First we remark that for $A=\mathrm{Id}$ the function $g$ leaves the quadrants invariant.  It is thus sufficient to consider the case where $p$ and $q$ vanish outside $\{x_i>0 \; \forall i\}$ and show that no solutions
of $s=g(s')$ exist under this condition. Using step 2 we can assume that $p_i(x_i)>0$ for all $x_i>0$. In the following all domain are assumed to 
be the positive half-line.
For $A=\mathrm{Id}$ the condition \eqref{eq:main_condition} becomes
for $x=(x_1,\ldots, x_d)$ such that $x_i>0$ 
\begin{align}
 \begin{split}
 0 
&=  4dx_ix_j -2|x|^2x_i\ln(p_j)'-2|x|^2x_j\ln(p_i)'-2|x|^2x_jx_i(\ln(p_i)'' + \ln(p_j)'')
 \\
 &\quad +
 \sum_k 8x_kx_ix_j\ln(p_k)'+
\sum_{k} 4x_k^2x_ix_j\ln(p_k)''
.
 \end{split}
\end{align} 
Dividing this by $2x_ix_j$ we obtain 
\begin{align}\label{eq:final_a=id}
0= 2d -|x|^2\left(\frac{\ln(p_j)'}{x_j} +\ln(p_j)'' + \frac{\ln(p_i)'}{x_i} + \ln(p_i)''\right)+
 \sum_k 4x_k\ln(p_k)'+2x_k^2\ln(p_k)''.
\end{align}
We now assume that $d>2$.
Let $i$, $j$, and $r$ be pairwise different. 
Using the last display with $i$, $j$ and $j$, $r$ and subtracting the resulting equations we obtain
\begin{align}
0=|x|^2\left(\frac{\ln(p_i)'}{x_i} +\ln(p_i)'' -\frac{\ln(p_r)'}{x_r} -\ln(p_r)'' \right).
\end{align}
Varying $x_r$ and $x_i$ independently and since $i\in [n]$ is arbitrary we conclude that there is a constant $\kappa$ 
such that 
\begin{align}
\frac{\ln(p_i)'}{x_i} + \ln(p_i)''=\kappa
\end{align}
for all $i$ and $x_i>0$.

\paragraph{Step 9: Conclusion for $n>2$.}
The solutions of the  ODE $y(x)/x+y'(x)=\kappa$ are given by
\begin{align}
\frac{\alpha}{x}+\frac{\kappa x}{2}
\end{align}
where $\alpha$ is any constant.
We conclude that there are constants $\alpha_j$ such that
\begin{align}\label{eq:finalpdensity}
\ln(p_j)(x) = \alpha_j \ln(x)-\kappa \frac{ x^2}{4} + c
\Rightarrow 
p_j(x) \propto x^{\alpha_j} e^{-\frac{-\kappa x^2}{4}}.
\end{align}
This implies that 
\begin{align}
q(y) = p(y/|y|^2)|y|^{-2d}
\propto \prod_j \frac{y_j^{\alpha_j}}{|y|^{2\alpha_j}}
e^{-\frac{- \kappa y_j^2}{4|y|^4}} |y|^{-2d}.
\end{align}
By applying the main argument to $q$ we infer that $q_j$ has to have again the same structure as in \eqref{eq:finalpdensity} so we conclude that 
$\kappa=0$  and $\sum_j \alpha_j=-d$.
Alternatively, one directly sees that $q$ only factorizes as $q(y)=\prod q_i(y_i)$
if those conditions hold. 
It is easy to see that those densities satisfy the 
assumptions. However, $x^\alpha$ is never integrable
so there are no probability distributions satisfying the relations $s=g(s')$
This ends the proof for $d>2$.
\paragraph{Step 10: Conclusion for $d=2$.}
For $n=2$ we cannot simplify \eqref{eq:final_a=id} by considering indices $i\neq j\neq k$.
Instead, we directly exploit \eqref{eq:final_a=id} to obtain a similar conclusion.
Similarly to the argument in Step 6 it can be shown that 
$\ln(p_r)''$ and $\ln(p_r)'/x_r$ are bounded for $x_r$ away from 0.
Then we consider $x_1\to \infty$ in \eqref{eq:final_a=id}
and divide by $x_1^2$. We get (using $\{i,j\}=\{1,2\}$)
\begin{align}
0=\left(\frac{\ln(p_1)'}{x_1} +\ln(p_1)'' + \frac{\ln(p_2)'}{x_2} + \ln(p_2)''\right)+
  4\frac{\ln(p_1)'}{x_1}+2\ln(p_1)'' + O(x_1^{-1}).
\end{align}
By varying $x_2$ we conclude just as for $d>2$ that
$\frac{\ln(p_2)'}{x_2} + \ln(p_2)''$ is constant. We 
conclude as before.
\end{proof}

Let us now prove the  geometric result from Lemma~\ref{le:geometry} above.
\begin{proof}
\begin{figure}
    \centering
    \includegraphics[width=\textwidth]{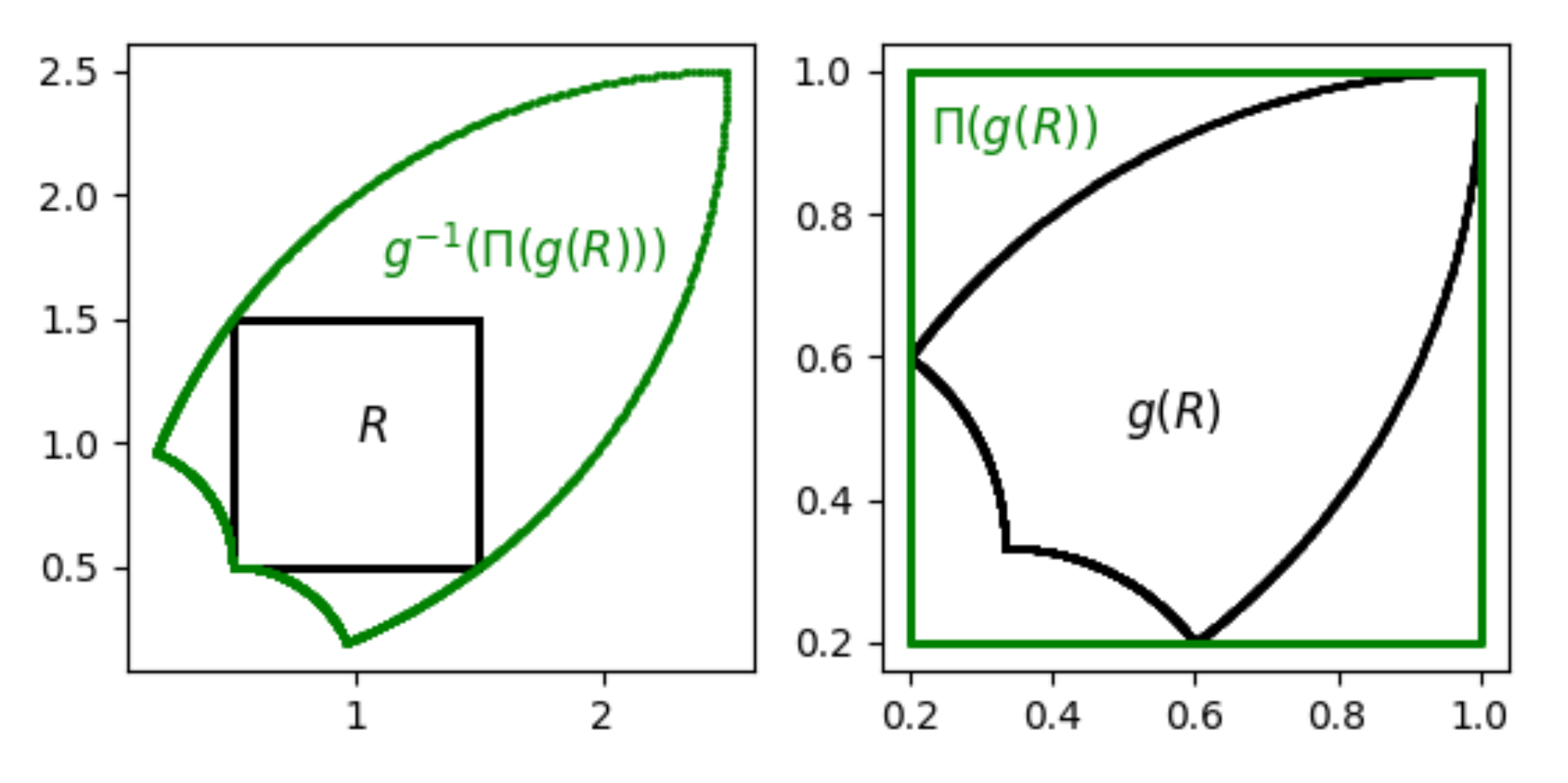}
    \caption{The black rectangle $R$ on the left is mapped by a conformal map to the
    black shape $g(R)$ on the right. When mapping the smallest rectangle $\Pi(g(R))$
    containing $g(R)$  (green rectangle on the right) back to $g^{-1}(\Pi(g(R)))$ (green shape on the left) we  obtain a larger set.}
    \label{fig:conformal}
\end{figure}

The main idea of the proof is that a box contained in $U$ after inversion is distorted so that its convex hull (contained in $O$) is strictly bigger than the box image
so that inverting backwards gives us a bigger box in $U$ except for some special cases.
An illustration of this argument is shown in Figure~\ref{fig:conformal}. The formal argument below is slightly technical.
An illustration of the actual argument can be found in Figure~\ref{fig:conformal2}.

To simplify the notation we write $\iota$ for the inversions $x\to x/|x|^2$.
Then we get $g= A\circ \iota=\iota\circ A$.
We consider the projections $\pi_i:\R^d\to\R $ projecting on the $i$-th coordinate.
We consider a map $\Pi$  on subsets of $\R^d$  defined by
$\Pi(M)=\pi_1(M)\times\ldots\times \pi_d(M)$. Let $\mc{C}$ denote the convex hull of a set. 
Let $R\subset O$ be a (connected) box.
  Then $ g(R)\subset U$ implies $\Pi(g(R))\subset \Pi(U)=U$.
Since $g(R)$ is connected $\Pi(g(R))$ is convex and
thus $\mc{C}(g(R))\subset \Pi(g(R))$.
We conclude that
$g^{-1}(\mc{C}(g(R))\subset g^{-1}(\Pi(g(R))
\subset g^{-1}(U)\subset O$.
As $A$ is linear we have $A^{-1}\mc{C}A(M)=\mc{C}(M)$ for any set $M\subset \R^d$.
Thus, we get $g^{-1}(\mc{C}(g(R))=\iota \mc{C}\iota(R)$.

\begin{figure}
    \centering
    \includegraphics[width=\textwidth]{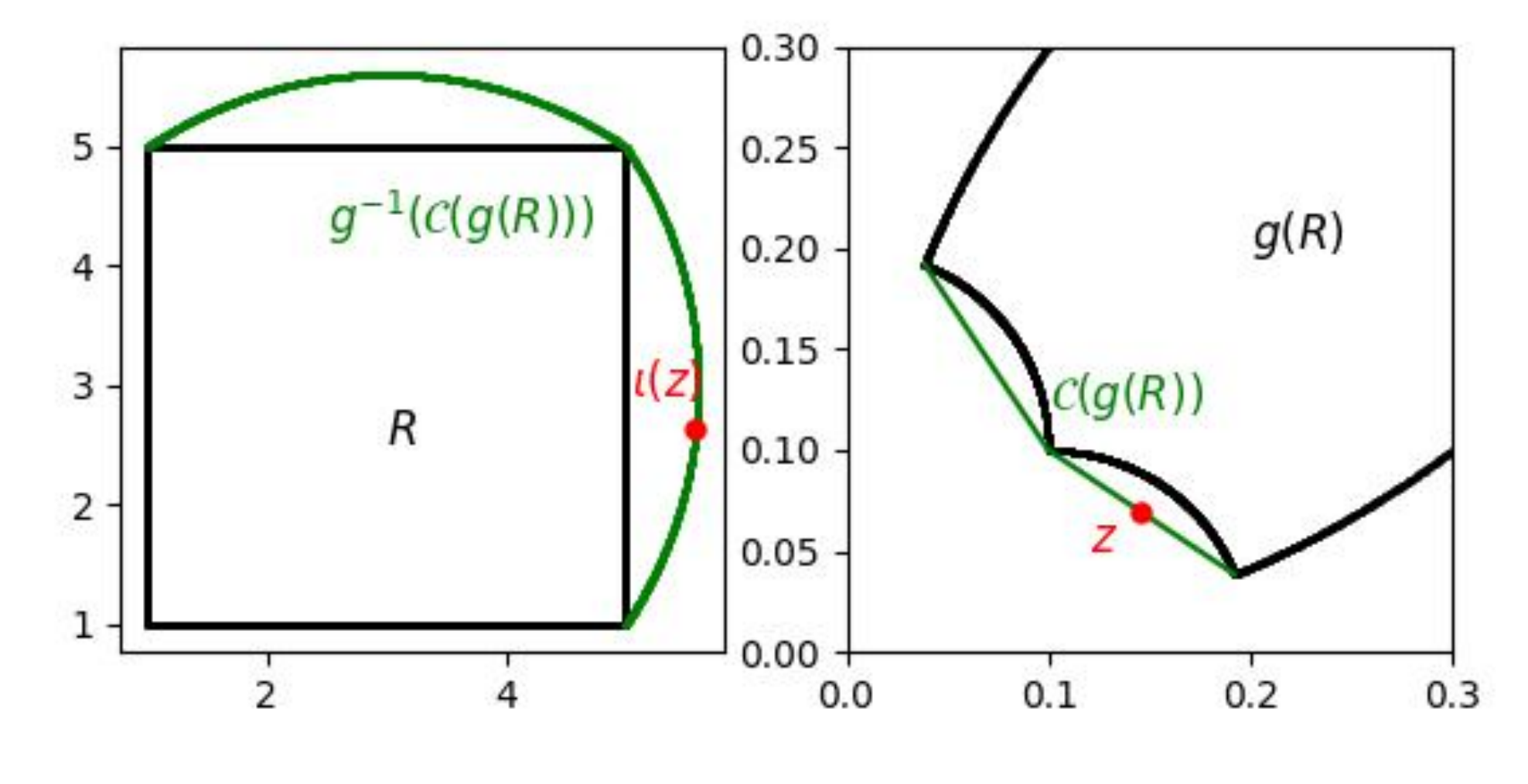}
    \caption{The black rectangle $R$ on the left is mapped by $\iota$  to the
    black shape $g(R)$ on the right. The green shape on the right shows the convex
    hull $\mc{C}(g(R))$. The point $z$ is defined as in the text 
    and its image $\iota(z)$ lies outside $R$.}
    \label{fig:conformal2}
\end{figure}
W.l.o.g.\ we now suppose that there is a box $R=(x_1,y_1)\times \ldots (x_d,y_d)\subset O$ with $0<x_i<y_i$. We write $x'=(x_2,\ldots, x_d)$, $y'=(y_2,\ldots, y_d)$. 
We consider the point
\begin{align}
    z = \frac12 \iota((y_1, x')^\top)+\frac12 \iota((y_1, y')^\top).
\end{align}
Clearly $z\in \mc{C}\iota R$.
Then
\begin{align}
\iota(z)= \iota\left(\frac12 \iota((y_1, x')^\top)+\frac12 \iota((y_1, y')^\top)\right) \in
 \iota \mc{C}\iota(R) 
\end{align}
We calculate
\begin{align}
\begin{split}
\iota\left(\frac12 \iota((y_1, x')^\top)+\frac12 \iota((y_1, y')^\top))\right)
&=\iota \left(\frac12 \frac{(y_1, x')^\top}{y_1^2+x'^2} 
+\frac12 \frac{(y_1,y')^\top}{y_1^2+y'^2} \right)
\\
&=\frac{\frac12 \frac{(y_1, x')^\top}{y_1^2+x'^2} 
+\frac12 \frac{(y_1,y')^\top}{y_1^2+y'^2}}{\left\lVert \frac12 \frac{(y_1, x')^\top}{y_1^2+x'^2} 
+\frac12 \frac{(y_1,y')^\top}{y_1^2+y'^2}\right\rVert^2}
\end{split}
\end{align}
We bound (using $x'^2<y'^2$)
\begin{align}
\begin{split}
\left\lVert \frac12 \frac{(y_1, x')^\top}{y_1^2+x'^2} 
+\frac12 \frac{(y_1,y')^\top}{y_1^2+y'^2}\right\rVert^2
&<
\frac12 \left\lVert  \frac{(y_1, x')^\top}{y_1^2+x'^2} \right\rVert^2
+
\frac12 \left\lVert \frac{(y_1,y')^\top}{y_1^2+y'^2} \right\rVert^2
\\ 
&=
\frac12\frac{1}{y_1^2+x'^2}+
\frac12\frac{1}{y_1^2+y'^2}.
\end{split}
\end{align}
Together the last displays imply that
\begin{align}
 \iota\left(\frac12 \iota((y_1, x')^\top)+\frac12 \iota((y_1, y')^\top))\right)_1
 >y_1.
\end{align}
Let $z_1$ be maximal such that $(x_1,z_1)\subset O_1 $.
Then the reasoning above shows that $z_1=\infty$.
The same reasoning for the other coordinates implies
that $R'=(x_1,\infty)\times\ldots\times(x_d,\infty)\subset O$.
By applying the same reasoning to sequences of boxes in $g(R)$ approaching the origin we conclude that $U$ is 
the union of quadrants (and the same holds for $O$).

It remains to prove the last remark.
 As quadrants are invariant 
under $\iota$ we have $\iota(O)=O$ and conclude $AO=U$, or equivalently $A^\top U=O$.
It is sufficient to show that $0\in U_i$.
For simplicity we assume $R=(0,\infty)^d\subset U$, the generalisation to 
other quadrants is immediate.
Since we assume that the $i$-th row $v_i = A^\top e_i$ of $A$  is not equal to a signed standard basis vector it  has at least two non-zero entries. Thus we can find
$w$ such that $w\cdot v_i=0$ and all entries of $w$ are non-zero.
Since $w$ is orthogonal to the span of $Ae_i$ there is a vector $\alpha$ such that
$A^\top \alpha=w$ and $\alpha_i=0$. By adding a suitable vector $\beta$ we can ensure that
$(\beta+\alpha)_i=0$, all entries of $A^\top(\alpha+\beta) $ are non-zero 
and $(\beta + \alpha)_j>0 $ for $j\neq i$. 
The second condition  can be satisfied by picking the entries of $\beta$ one after another. The conditions $(\beta + \alpha)_j>0 $ for $j\neq i$ and $(\beta + \alpha)_i=0$
imply that $\beta+\alpha \in \overline{U}$ since we assumed $(0,\infty)^d\subset U$.
But then $A^\top(\alpha+\beta)\in \overline{O}$ and since
$A^\top(\alpha+\beta)$ is strictly contained in a quadrant (all entries are non-zero) we conclude that $A^\top(\alpha+\beta)\in O$ and thus $\alpha+\beta\in U$.
This implies $(\alpha+\beta)_i=0\in U_i$.

\end{proof}

\subsection{Proof of Corollary~\ref{co:rescale}}
Here we prove the simple extension of Theorem~\ref{th:conformal} to rescaled conformal maps.
\begin{proof}[Proof of Corollary~\ref{co:rescale}]
Suppose that there are $f,f'\in \mc{F}_{\sconf}$ and $\P,\P'\in \mc{P}_\conf$
such that 
\begin{align}\label{eq:cor_start}
    f_\ast \P\indistribution f'_\ast\P'.
\end{align} 
Then we need to show that 
$f^{-1}f'\in \mc{S}_\rep$. By definition of $\mc{F}_\sconf$
there are $g,g'\in \mc{F}_\conf$ and $h, h'\in \mc{S}_{\rep}$
such that $f=g\circ h$, $f'=g'\circ h'$.
The smoothness condition for $\mc{F}_\sconf$ implies $h,h'$ are three times differentiable and thus $\Q=h_\ast \P$ and $\Q'= h'_\ast\P'$ satisfy the smoothness condition
of measures in $ \mc{P}_\conf$. 
From \eqref{eq:cor_start} we infer 
$g_\ast \Q\indistribution g'_\ast \Q'$ and this is the observational distribution.
By the assumption on the observational distribution, we conclude that $\Q \sim g^{-1}(x)$ and $\Q'\sim g'^{-1}(x)$ have at most one Gaussian component and thus $\Q,\Q'\in \mc{P}_\conf$.
 Applying Theorem~\ref{th:conformal}
we infer that $g^{-1}g'\in \mc{S}_\conf\subset \mc{S}_{\rep}$. 
Since $\mc{S}_{\rep}$ is a group, we conclude that
$f^{-1}f'=h^{-1}\circ(g^{-1}g')\circ h'\in \mc{S}_\rep$. This ends the proof.
\end{proof}

\subsection{Proof of Theorem~\ref{th:conf2}}
Now we address the case $d=2$. 
Let us first note that conformal maps $f:\Omega\to \R^2$ for $\Omega\subset\R^2$ can be identified
with holomorphic maps $g:\Omega \to \mathbb{C}$ with non-vanishing derivatives
by considering $g(x+iy)= f_1(x+iy)+if_2(x+iy)$.
Moreover, 
\begin{align}\label{eq:det_holom}
    |\det Df|=|g'|^2
\end{align}
where $g'$ denotes the complex derivative.
The proof of Theorem~\ref{th:conf2} is based on the fact that conformal maps between rectangles can be characterized rather explicitly.
In particular, we use the following result.
\begin{theorem}[Schwarz-Christoffel mapping]\label{th:schwarz}
Conformal maps $f$ that map the unit disk $\mathbb{D}=\{z\in \mathbb{C}| \, |z|<1\} $
to a polygon with angles $\alpha_k\pi$ for $1\leq k\leq n$
can be written as
\begin{align}\label{eq:schwarz}
    f(z) = C \int_0^z \prod_{k=1}^n (w - w_k)^{-(1-\alpha_k)}\, \d w + K
\end{align}
for two constants $C,K\in \mathbb{C}$ and $w_k\in \mathbb{C}$
with $|w_k|=1$.
\end{theorem}
A proof of this result can be found in any textbook on complex analysis,
e.g., \cite{Ahlfors1966}. 
With this result we can prove Theorem~\ref{th:conf2}.
\begin{proof}[Proof of Theorem~\ref{th:conf2}]
Assume there are two probability distribution $\P_1\in \mc{P}_{\conf2}$ and $\P_2\in \mc{P}_{\conf 2}$ with densities $p_1$ and $p_2$
supported on finite rectangles $R_1$ and $R_2$ and a conformal map
$f:R_1\to R_2$ such that
$f_\ast \P_1=\P_2$. Let $h:\mathbb{D}\to R_1$ be a conformal map
which by the Riemann mapping theorem exists and
by Theorem~\ref{th:schwarz} can be expressed as in \eqref{eq:schwarz}, i.e.,
\begin{align}\label{eq:def_h_conformal}
    h(z)= C \int_0^z \prod_{k=1}^4 (z - w_k)^{-1/2}\,\d z+K
\end{align}
where $|w_k|=1$ and we used that all angles in a rectangle are $\pi/2$.
The points $w_k$ are the preimages of the corners of $R_1$.
Then the derivative of $h$ can be written as
\begin{align}
    h'(z)= C  \prod_{k=1}^4 (z - w_k)^{-1/2}.
\end{align}
Note that $h'$ is bounded above and below away from the points $w_k$
and $|h'(x)|\to \infty$ as $x\to w_k$.
We can write $f=(f\circ h)\circ h^{-1}$ where 
$g=f\circ h$ is a conformal map from $\mathbb{D}$ to $R_2$.
This can be written again as in \eqref{eq:schwarz} with 
points $w_1', \ldots, w_4'$ which are the preimages of the corners of $R_2$, i.e.,
\begin{align}\label{eq:def_g_conformal}
    g(z)= C'  \int_0^z\prod_{k=1}^4 (z - w_k')^{-1/2}\, \d z+K'
\end{align}
The condition $f_\ast \P_1=\P_2$ implies that 
\begin{align}\label{eq:density_complex}
    p_1(x)=p_2(f(x))|f'(x)|^2
\end{align}
(using the relation \eqref{eq:det_holom} between complex derivative and maps from $\R^2\to\R^2$).
Now we infer 
\begin{align}\label{eq:rect_conformal_diff}
|f'(x)|^2=
|(f\circ h\circ h^{-1}(x)|^2=|(f\circ h)'(h^{-1}(x))|^2 |(h^{-1})'(x)|^2
=\frac{|g'(y)|^2}{
|h'(y)|^2}
\end{align}
where $y=h^{-1}(x)$.
Applying this with $y\to w_k$ we get $h'(y)\to \infty$.
Using the assumption that the densities of $\P_1$ and 
$\P_2$ are bounded above and below we infer from \eqref{eq:density_complex}
and \eqref{eq:rect_conformal_diff} that
then $|g'(y)|\to \infty$ from which we conclude $y\to w_{k'}'$ for 
some $k'$. This implies $w_k=w'_{k'}$. 
After renaming the corners, we get $w_k=w_k'$ for all $k$.
From \eqref{eq:def_h_conformal} and \eqref{eq:def_g_conformal} we see that 
\begin{align}
C'^{-1}(g(z)-K')=C^{-1}(h(z)-K). 
\end{align}
Thus $g(z)=C'h(z)/C -KC'/C+K'$ so we conclude that
\begin{align}
(g\circ h^{-1})(h(z))=g(z)=C'h(z)/C -KC'/C+K'.
\end{align}
In particular there are  constants $A,B\in \mathbb{C}$ such that 
$f(x)=(g\circ h^{-1})(x)=Ax+B$.
From $f(R_1)=R_2$ we infer that $A\in \mathbb{R}\cup i\R$. This ends the proof.
\end{proof}
Let us finally remark that the identifiability of conformal maps for 
distributions with full support in $d=2$ follows just as in $d\geq 3$
because every bijective conformal map of the Riemann sphere to itself is a Moebius transformation so we can apply Lemma~\ref{le:geometry}.
We expect that the result can be extended to more general densities using the same strategy and a more careful analysis of the density close to the boundary.

\section{Proofs for the results on  OCTs}\label{app:ima}

In this section we collect the missing proofs for Section~\ref{sec:ima}.

\subsection{Proof of Proposition~\ref{prop:non-uniqueness}}
First we prove Proposition~\ref{prop:non-uniqueness}. We refer to Appendix~\ref{app:spurious} for a general review and characterisation of spurious solutions.
\begin{proof}[Proof of Proposition~\ref{prop:non-uniqueness}]
The proof essentially shows that polar coordinates are an example of such a function.
In dimension $d$ they are defined by $\Phi:(a,b)\times [0,\pi] \times [0,\pi/2]^{d-2}\to \R^d$ with
\begin{align}
\begin{split}
x_1 &= r \sin(\p)\sin(\theta_1)\ldots \sin(\theta_{d-2})
\\
x_2 &= r \cos(\p)\sin(\theta_1)\ldots \sin(\theta_{d-2})
\\
x_3 &= r \cos(\theta_1)\ldots \sin(\theta_{d-2})
\\
\vdots\; &\,\vdots \qquad \vdots
\\
x_d &= r\cos(\theta_{d-2}). 
\end{split}
\end{align}
The following lemma is well known.
\begin{lemma}\label{le:polar}
The polar coordinates $\Phi$ defined above satisfy:
\begin{enumerate}
\item
The transformation satisfies $\Phi\in \mc{F}_{\ima}$.
\item The determinant of the Jacobian is given by
\begin{align}
\det D\Phi_d = r^{d-1}\sin{\theta_1}\sin(\theta_2)^2\ldots \sin(\theta_{d-2})^{d-2}.
\end{align}
\item The image of $\Phi_d$ is, up to a set of measure zero, the annulus $\{x\in \R^n: a<|x|<b\}$.
\end{enumerate}
\end{lemma}
\begin{proof}
For the first part we only need to show that $D\Phi$ has orthogonal columns
which can formally be done by induction noting that
\begin{align}
\Phi_d(r, \p,\theta_1,\ldots,\theta_{d-2}) = (\Phi_{d-1}
(r, \p,\theta_1,\ldots,\theta_{d-3})\sin(\theta_{d-2}), r\cos(\theta_{d-2})).
\end{align}
The determinant and the image can also be derived from this recursion.
\end{proof}
We now define for $k\in \mathbb{N}$ the functions $g_k:(0, \pi/2)\to \R$
\begin{align}
g_k(\theta) = \int_0^\theta \sin(t)^k \,\mathrm{d} t.
\end{align}
Clearly $g_k'(\theta) = \sin(\theta)^k$.
They are strictly increasing functions with positive derivative, i.e., diffeomorphisms, so we can define their inverses on an open  interval $h_k:I_k\to (0,\pi/2)$ which are also differentiable functions.
Note that 
\begin{align}
h_k'(\alpha)= (g_k'(h_k(\alpha))^{-1}=\sin(h_k(\alpha))^{-k}.
\end{align}
Denote the density of $\P$ by $p$. By assumption $\P$ is invariant under rotations so we can write with a slight abuse of notation $p(s)=p(|s|)$.
We assume that $p(|s|)$ is positive on a, possibly unbounded, interval $(a,b)$.
We consider $q:\R_+\to \R$ given by
\begin{align}\label{eq:def_q_density}
    q(r)= d\omega_d r^{d-1} p(r)
\end{align}
where $\omega_d$ is the volume of the unit ball in dimension $d$ and 
these constants ensure that $q$ is a probability density. 
Define $F_q:(0, \infty)\to [0,1]$ as the cdf for the probability density $q$, i.e., 
\begin{align}
F_q(r)=\int_0^r q(z)\,\d z.
\end{align}
Since $q(z)>0$ iff
$t\in (a,b)$ we conclude that $F_q$ restricted to $(a,b)$ is a continuous and strictly increasing from 0 to 1 and has a positive derivative. 
Hence, we can define $\psi:(0, 1)\to (a,b)$ by $\psi(t)=F_q^{-1}(t)$
and $\psi$ is differentiable with 
\begin{align}
    \psi'(t)=(F_q'(\psi(t)))^{-1}=\frac{1}{q(\psi(t))}.
\end{align} 
We define the domains $D_d=(0,1)\times (0,2\pi)\times I_1\times\ldots I_{d-2}$
Now we consider the map $h:D_d\to (a,b)\times (0,2\pi)\times (0,\pi)^{d-2}$ given by
\begin{align}
h(t,\p,\alpha_1,\ldots,\alpha_{d-2})= (\psi(t),\p, h_1(\alpha_1),\ldots,h_{d-2}(\alpha_{d-2})).
\end{align}
Note that $h$ is a coordinate-wise transformation and the determinant of its Jacobian is given by
\begin{align}
\det Dh = \frac{1}{q(\psi(t))}\prod_{k=1}^{d-2} \sin(h_k(\alpha))^{-k}.
\end{align}
Moreover Lemma~\ref{le:polar} gives
\begin{align}
\det (D\Phi)(h(t,\p,\alpha_1,\ldots,\alpha_{d-2}))
= \psi(t)^{d-1}\sin(h_1(\alpha_1))\sin(h_2(\alpha_2))^2\ldots \sin(h_{d-2}(\alpha_{d-2}))^{d-2}
\end{align}
and thus 
\begin{align}\label{eq:det_phi_h}
\det D(\Phi\circ h) (t,\p,\alpha_1,\ldots,\alpha_{d-2})=\frac{\psi(t)^{d-1}}{q(\psi(t))}.
\end{align}
We now consider the probability measure $\mathbb{Q}=\mc{U}(0,1)\otimes \mc{U}((0,2\pi))
\otimes \mc{U}(I_1)\otimes \ldots \otimes \mc{U}(I_{n-2})$
where $\mc{U}$ denotes the uniform distribution on an interval.
Denote by $\lambda:\R^d\to\R^d$ the scaling function $\lambda(s)= (s_1,2\pi s_2, 
|I_1|s_3,\ldots, |I_{d-2}|s_d)$. Clearly
$\lambda$ maps $C_d$ to $(0,1)\times(0,2\pi)\times I_1\times\ldots\times I_{d-1}$
and $\lambda_\ast \nu=\Q$.
We now consider the measure 
\begin{align}
    (\Phi\circ h\circ \lambda)_\ast \nu =
    (\Phi\circ h)_\ast \Q.
\end{align}
We denote its density by  $\tilde{p}$ and obtain using \eqref{eq:density_pushforward_2},
\eqref{eq:det_phi_h},
and that the density of $\Q$ is constant
\begin{align}
\begin{split}
\tilde{p}(x) &\propto 
|\Det D(\Phi\circ h)^{-1})(x))|
	\\
	&=\frac{1}{ |\Det D(\Phi\circ h)((\Phi\circ h)^{-1}(x))|}
	\\
	&=
	\frac{q\big(\psi((\Phi\circ h)^{-1}_1(x))\big)}{
	\psi((\Phi\circ h)^{-1}_1(x))^{d-1}}.
	\end{split}
\end{align}
Note that by definition of $\Phi$ we have $\Phi^{-1}(x)_1=|x|$
and $h$ acts coordinatewise where the action on the first coordinate is $\psi$ so that
\begin{align}
\psi((\Phi\circ h)^{-1}(x)_1)=\psi((h^{-1}\circ \Phi^{-1})(x)_1)
=\psi(\psi^{-1}(|x|))=|x|.
\end{align}
We conclude, 
using \eqref{eq:def_q_density} and the last display, that
\begin{align}
  \tilde{p}(x) \propto \frac{q(|x|)}{|x|^{d-1}}\propto p(|x|).
\end{align}
This shows 
\begin{align}
    \P=(\Phi\circ h\circ \lambda)_\ast \nu.
\end{align}
It remains to be shown that $\Phi\circ h\circ \lambda\in \mc{F}_\ima$.
But we have seen in Lemma~\ref{le:polar} that $\Phi\in \mc{F}_\ima$
and $h\circ \lambda$ is an invertible coordinate-wise transformation
so that $h\circ \lambda\in \mc{F}^R_\ima$ (see \eqref{eq:FR_ima}) implying
$\Phi\circ h\circ \lambda\in \mc{F}_\ima$. Note that we here essentially use the fact that precomposition with a function that acts on each coordinate separately preserves orthogonality of the columns of the Jacobian. 
\end{proof}

\subsection{Proof of Theorems~\ref{th:loc_loc}}
We now consider smooth deformations of a data generating mechanism $x=f(s)$.
For this it is helpful to phrase these as flows generated by vector fields.
For a brief review of these notions we refer to Appendix~\ref{app:math}
and for an extensive introduction we refer to any textbook on differential geometry.
We now give a complete proof of Theorem~\ref{th:loc_loc}.
\begin{proof}[Proof of Theorem~\ref{th:loc_loc}]
To simplify the notation we first focus on the case where $M=\R^d$.
The necessary modifications for the manifold case are indicated afterwards.
We define 
\begin{align}
    \Psi_t = (\Phi_t)^{-1}\circ f_t
\end{align} so that $\Psi_t:\R^d\to\R^d$ is a diffeomorphism for every $t$ and $\Psi_0(s)=s$ by assumption.
We denote the vector field that generates $\Psi_t$ by 
$X:(0,1)^d\to \R^d$, i.e., $X$ satisfies 
\begin{align}
\partial_t\Psi_t(x)=X_t(\Psi_t(x)).
\end{align}
The assumption $(\Phi_t)_\ast \nu = f_\ast \nu$ implies 
$\nu = (\Psi_t)_\ast f_\ast \nu$. Then the continuity equation \eqref{eq:continuity}
implies that $\Div(X_t)=0$ on $(0,1)^d$.
By assumption, $\Phi_t\in \mc{F}_\ima$, which means that
\begin{align}\label{eq:lame2}
(D\Phi_t)^\top D\Phi_t=\Lambda_t
\end{align}
 where $\Lambda_t:(0,1)^d\to \mathrm{Diag}(d)$
maps to diagonal matrices. 
Similarly $f_t\in \mc{F}_{\ima}$ implies that there is a function
$\Omega_t:(0,1)^d\to \mathrm{Diag}(d)$ such that
\begin{align}
    (Df_t)^\top Df_t=\Omega_t.
\end{align}
We now evaluate $\partial_t \Omega_t$ in terms of $\Psi_t$ and $\Phi_t$.
To evaluate the time derivative it is convenient to write
$\Psi(t,s)$ instead of $\Psi_t(s)$.
We get, using $f_t(s)= \Phi(t, \Psi(t,s))$,
\begin{align}
\begin{split}\label{eq:dt_D_f}
    \partial_t Df_t(s)
    &= \partial_t \left((D \Phi)(t,\Psi(t,s))(D\Psi)(t,s)\right)
    \\
    &= (\partial_t D\Phi)(t, \Psi(t,s)) (D\Psi)(t,s)
    + \sum_{k=1}^d (\partial_t\Psi)_k(t,s)(\partial_k D\Phi)(t,\Psi(t,s)) (D\Psi)(t,s)
   \\
   &\quad+ (D \Phi)(t,\Psi(t,s))(D\partial_t\Psi)(t,s)
    \\
    &=
    (\partial_t D\Phi_t)( \Psi_t(s)) (D\Psi_t)(s)
    + \sum_{k=1}^d (X_t)_k(\Psi_t(s))(\partial_k D\Phi_t)(\Psi_t(s))(D\Psi_t)(s)
    \\
    &\quad+ (D \Phi_t)(\Psi_t(s))(D (X_t(\Psi_t(s))
    \\
    &=
   \left( \left(
    (\partial_t D\Phi_t)
    + \sum_{k=1}^d (X_t)_k(\partial_k D\Phi_t)
    +(D \Phi_t)(D X_t)
    \right)\circ \Psi_t(s)\right) (D\Psi_t)(s).
    \end{split}
\end{align}
Note that also
\begin{align}\label{eq:Df_t}
    Df_t(s)=(D\Phi_t)(\Psi_t(s))(D\Psi_t)(s).
\end{align}
Combining this with the last display we get (dropping the positional argument $s$
for conciseness)
\begin{align}
\begin{split}\label{eq:partial_Omega}
    (D\Psi_t)^{-\top}& (\partial_t \Omega_t)(D\Psi_t)^{-1}
   =
    (D\Psi_t)^{-\top} \left(
    Df_t^\top (\partial_t Df_t) + (\partial_tDf_t)^\top Df_t \right)(D\Psi_t)^{-1}
    \\
    &=
   \left[ D\Phi_t^\top \left(
    (\partial_t D\Phi_t)
    + \sum_{k=1}^d (X_t)_k(\partial_k D\Phi_t)
    +(D \Phi_t)(D X_t)
    \right)\right.
    \\
    &\quad \quad
 \left. +\left(
    (\partial_t D\Phi_t)
    + \sum_{k=1}^d (X_t)_k(\partial_k D\Phi_t)
    +(D \Phi_t)(D X_t)\right)^\top   D\Phi_t \right]
    \circ \Psi_t
    \\
    &=
   \Bigg[ D\Phi_t^\top(\partial_t D\Phi_t)
   + (\partial_t D\Phi_t)^\top D\Phi_t
   +D\Phi_t^\top D\Phi_t DX_t
   +DX^\top_t D\Phi_t^\top D\Phi_t 
   \\
   &\quad\quad  
    + \sum_{k=1}^d (X_t)_k\left((\partial_k D\Phi_t)^\top D\Phi_t
    + D\Phi_t^\top (\partial_k D\Phi_t)
    \right)\Bigg]\circ \Psi_t
    \\
    &=
    \left(\partial_t\Lambda_t + \Lambda_t DX_t +DX^\top_t \Lambda_t
    + \sum_k (X_t)_k \partial_k\Lambda_t\right)\circ \Psi_t.
 \end{split}   
\end{align}
We now consider $t=0$ and drop the $t$ argument from the notation. By assumption $\Psi_0=\mathrm{Id}$ so the equation simplifies to
\begin{align}\label{eq:main_pde}
    \partial_t \Omega
    =
    \partial_t \Lambda +\Lambda DX + DX\Lambda +\sum_k X_k\partial_k\Lambda.
\end{align}
Now we use that $\Lambda_t$ and $\Omega_t$ 
map to diagonal matrices for all $t$, in particular $\partial_t \Lambda$
and $\partial_t \Omega$ are diagonal matrices.
We conclude that for $i\neq j$ the equation 
\begin{align}\label{eq:diff_first_order}
\begin{split}
0&=(\Lambda DX)_{ij}+((DX)^\top \Lambda)_{ij}
=
\Lambda_{ii} (DX)_{ij}+
 (DX)_{ji}\Lambda_{jj}
\\
&= \Lambda_{ii}\partial_j X_i+ \Lambda_{jj}\partial_i X_j
\end{split}
\end{align}
holds. Thus, we obtain  a system of first order Partial Differential Equations (PDE) for $X_{t=0}$.
We now write $\Lambda_j=\Lambda_{jj}$.
We also fix an $i\in \{1,\ldots, d\}$ in the following.
Then we can rewrite \eqref{eq:diff_first_order} concisely as
\begin{align}\label{eq:diff_first_order_short}
    \Lambda_i \partial_j X_i + \Lambda_j \partial_i X_j=0\quad \text{for $i\neq j$.}
\end{align}
We divide equation \eqref{eq:diff_first_order_short} by $\Lambda_j$ apply $\partial_j$
and sum over $j\neq i$ to obtain
\begin{align}
\sum_{j\neq i} \partial_j\left(\frac{\Lambda_i}{\Lambda_j}\partial_j  X_i\right)=-\sum_{j\neq i}
\partial_j\partial_i X_j
= -\partial_i \Div X + \partial_i^2 X_i
= \partial_i^2 X_i.
\end{align}
This implies that $X_i$ satisfies the wave equation
\begin{align}\label{eq:pde2}
\partial_i^2 X_i - \sum_{j\neq i} \partial_j(a_j\partial_j X_i)=0 \text{  on $(0,1)^d$}
\end{align}
where $a_j= \Lambda_i/\Lambda_j$. Note that by assumption $a_j\in C^1((0,1)^d)$
and $a_j$ is positive because we assumed that $\Phi_t$ are diffeomorphisms implying
$\Lambda_{j}>0$ (because $D\Phi$ is invertible).
Now we use the assumption that 
$\Phi_t(s)=f_t(s)$ if  $\mathrm{dist}(s,\partial\cube)<\varepsilon$.
This implies for such $s$ that $s=(\Phi_t^{-1}\circ f_t)(s)=\Psi_t(s)$, i.e., 
$\Psi_t(s)$ is constant.
We conclude that $0=\partial_t \Psi_t(s)=X_t(\Psi_t(s))=X_t(s)$
for all $s$ satisif=ying
$\mathrm{dist}(s,\partial\cube)<\varepsilon$.

Now we claim that this together with the PDE \eqref{eq:pde2}
implies that $X_i$ vanishes everywhere.
We set $\varepsilon'=\varepsilon/d$. 
Then we have 
$\mathrm{dist}(s,\partial\cube)<\varepsilon$ for all $s\notin (\varepsilon', 1-\varepsilon')^d$,
Thus $X_i$ solves the PDE \eqref{eq:pde2} on $(\eps', 1-\eps')^d$ with vanishing 
boundary data and vanishing derivatives at the boundary. 
Then the uniqueness of solutions for the Cauchy problem for hyperbolic PDE of second order 
which we stated in Theorem~\ref{th:hyperbolic} below implies that there is at most one
solution. 
Note that the ellipticity condition in \eqref{eq:elliptic} follows by noting
that the functions $a_j$ are continuous and  positive, and thus $\min_{(\eps, 1-\eps)^d}a_j>0$.

Since $X_i=0$ clearly solves the PDE, we conclude that $X_i=0$.
This argument applies to all  $i$ so we conclude that $X_{t=0}=0$.
Note that if $f_t=f_0$ is constant (this case is state in Corollary~\ref{co:undeformed}) the
left hand side of \eqref{eq:partial_Omega} vanishes for all $t$. Then we can infer with the same argument that $X_t=0$ for all $t$ and thus $\Psi_t=\mathrm{Id}$ and $\Phi_t=f$.

We now proceed with the general case. Since by assumption $\Psi_t$ is analytic in $t$
it is sufficient to show that $\partial_t^{l+1}\Psi_{t=0}(s)=0$
for all $l\in \mathbb{N}$
which implies that $\Psi_t(s)=\Psi_0(s)=s$.
We denote $X^{(l)}_t=\partial_t^l X_t$. By definition of $X_t$ and the chain rule it is easy to see
that  
$X^{(l')}_{t=0}=0$ for $0\leq l'\leq l$ implies $\partial_t^{l+1}\Psi_{t=0}(s)=0$. We now show this by induction. 
We apply $\partial_t^l$
to \eqref{eq:partial_Omega} at $t=0$. We obtain
(using $\Psi_0=\mathrm{Id}$)
\begin{align}
    \partial_t^{l+1} \Omega_t
    =
    \partial_t^{l+1} \Lambda_t +\Lambda DX_t^{(l)}
    +(DX_t^{(l)})^\top \Lambda_t +
    \sum_k (X_0^{(l)})_k \partial_k \Lambda_t
    + R(X_t,\ldots , X^{(l-1)}_t)
\end{align}
where $R$ denotes a remainder term where each summand contains a factor
of the form $X^{(l')}_{t=0}$ for some $0\leq l'<l$. Indeed, every time we 
differentiate a $\Psi_t$ term we get a $X_t$ term.
By the induction hypothesis $X^{(l')}_{t=0}=0$ for $l'<l$ and therefore 
$R=0$.
We conclude that $X^{(l)}$ satisfies \eqref{eq:diff_first_order_short}
just as $X=X^{(0)}$. Differentiating $\Div X=0$ with respect to $t$
we also conclude $\Div X^{(l)}=0$ and as before we conclude that
$X^{(l)}(s)=0$ for $s\notin (\eps',1-\eps')^d$. As above this implies $X^{(l)}_{t=0}(s)=0$.
This ends the proof if $M=\R^d$.

{
We now discuss the necessary extensions for the general case.
So we now assume that $f_t, \Phi_t:\R^d\to M$ for
some Riemannian manifold $(M,g)$.
The main strategy is to establish that \eqref{eq:main_pde} holds
where the definitions of the involved quantities are adapted suitably. 
Then we can conclude the proof as above.
Note that if we could find orthonormal coordinates locally the same proof applies.
However, this is general not possible so we need to argue more carefully.
We define $\Psi_t:\R^d\to\R^d$ and $X_t$ as above, i.e., $\Psi_t=(\Phi_t)^{-1}\circ f_t$ and $\partial_t \Psi_t=X_t(\Psi_t)$.
Then we define the matrix functions $\Lambda_t, \Omega_t\in \R^{d\times d}$ by
\begin{align}
    (\Lambda_t)_{ij}(s) &= \langle (d\Phi_t)(s) e_i, (d\Phi_t)(s)e_j\rangle_g,
    \\
     (\Omega_t)_{ij}(s) &= \langle (df_t)(s) e_i, (df_t)(s)e_j\rangle_g.
\end{align}
By assumption $\Lambda_t$ and $\Omega_t$ are diagonal.
We consider a (inverse) chart $\eta: U\to M$ where $U\subset \R^d$.
We now argue locally on $\eta(U)\subset M$ but we do not denote domain restriction of the function to improve readability.
We define 
\begin{align}
    \tilde{\Phi}_t&=\eta^{-1}\circ \Phi_t, 
    \\
    \tilde{f}_t &= \eta^{-1}\circ f_t.
\end{align}
Maps $\tilde{\Phi_t}$ and $\tilde{f_t}$ map $\R^d$ to itself and we can consider their usual derivatives.  Moreover, we have $\Psi_t=(\tilde\Phi_t)^{-1}\circ \tilde{f}_t$.
Observe that (using $\eta\tilde{\Phi}=\Phi$) we can rewrite $\Lambda_t$
using the induced metric through $\eta$ on $\R^d$
\begin{align}
\begin{split}\label{eq:rewrite_Lambda}
     (\Lambda_t)_{ij}(s)
    & =
     \langle (d\eta)(\tilde{\Phi}_t(s)) (d\tilde{\Phi}_t(s)) e_i
     ,(d\eta)(\tilde{\Phi}_t(s)) (d\tilde{\Phi}_t(s)) e_j\rangle_g
     \\
     &=
     \langle  (d\tilde{\Phi}_t(s)) e_i
     ,(d\eta)^\ast(\tilde{\Phi}_t(s))(d\eta)(\tilde{\Phi}_t(s)) (d\tilde{\Phi}_t(s)) e_j\rangle_{\R^d}
     \\
     &=
     \langle  D\tilde{\Phi}_t(s) e_i
     ,G(\tilde{\Phi}_t(s)) D\tilde{\Phi}_t(s) e_j\rangle_{\R^d}
     \end{split}
\end{align}
where we defined $G:U\to \R^{d\times d}$ by $G(z)=(d\eta)^\ast(z)(d\eta)(z)$ which captures the pullback metric on $\R^d$.
We can express this concisely and get similarly for $\Omega_t$
\begin{align}
     \Lambda_t(s)&=
     D\tilde{\Phi}_t(s)^\top G(\tilde{\Phi}_t(s)) D\tilde{\Phi}_t(s),
   \\ \label{eq:Omega_t_M}
   \Omega_t(s)
    &=
   D\tilde{f}_t(s)^\top G(\tilde{f}_t(s)) D\tilde{f}_t(s).
\end{align}
We now consider as above $\partial_t \Omega_t$ and get using the last display
\begin{align}
    \begin{split}
        \label{eq:dt_Omega_M}
    \partial_t   \Omega_t
    &=
  (\partial_t D\tilde{f}_t)^\top
     G(\tilde{f}_t) D\tilde{f}_t
     +
      D\tilde{f}_t^\top
     (\partial_t G(\tilde{f}_t)) D\tilde{f}_t 
     +    D\tilde{f}_t^\top
     G(\tilde{f}_t) (\partial_t D\tilde{f}_t ).
    \end{split}
\end{align}
Next we calculate
\begin{align}\label{eq:d_metric}
    \partial_t (G\circ \tilde{f}_t)
    = \partial_t (G\circ \tilde\Phi_t\circ \Psi_t)
    =
    \partial_t (G\circ \tilde\Phi_t) \circ \Psi_t
    + \sum_k \left((X_t)_k\circ \Psi_t\right) \cdot (\partial_k (G\circ \tilde\Phi_t))\circ \Psi_t 
\end{align}
Using the relations \eqref{eq:dt_D_f}, \eqref{eq:Df_t} (for $\tilde{f}$) and \eqref{eq:d_metric}
in \eqref{eq:dt_Omega_M} we obtain
\begin{align}\label{eq:manifold_almost_final}
  (D\Psi_t)^{-\top}  \partial_t \Omega_t (D\Psi_t)^{-1} = A_1 + A_2 + A_3,
\end{align}
where
\begin{align}
\begin{split}
    A_1
    &=
  \Bigg[  (\partial_t D\tilde\Phi_t)^\top (G\circ \tilde\Phi_t) D\tilde{\Phi}_t
    +
( D\tilde\Phi_t)^\top (G\circ \tilde\Phi_t) (\partial_tD\tilde{\Phi}_t)
+
    ( D\tilde\Phi_t)^\top (\partial_t(G\circ \tilde\Phi_t)) D\tilde{\Phi}_t
    \Bigg] \circ \Psi_t
    \\
    &=
    (\partial_t \Lambda_t)\circ \Psi_t,
    \\
    A_2
    &=
      \Bigg[  DX_t^\top ( D\tilde\Phi_t)^\top (G\circ \tilde\Phi_t) D\tilde{\Phi}_t
    +
( D\tilde\Phi_t)^\top (G\circ \tilde\Phi_t) D\tilde{\Phi}_t DX_t
    \Bigg] \circ \Psi_t
   \\
   &= (DX_t^\top \Lambda_t + \Lambda_t DX_t)\circ \Psi_t,
    \\
    A_3&=
   \bigg[ \sum_{k=1}^d (X_t)_k\left((\partial_k D\tilde\Phi_t)^\top
    (G\circ \tilde\Phi_t) (D\tilde\Phi_t)
    +
     ( D\tilde\Phi_t)^\top
   ( G\circ \tilde\Phi_t) (\partial_k D\tilde\Phi_t)\right)
\\
&\qquad \quad
  +( D\tilde\Phi_t)^\top \sum_k (X_t)_k\cdot (\partial_k (G\circ \tilde\Phi_t))
 D\tilde\Phi_t \bigg]\circ \Psi_t 
    \\
    &=\left(
    \sum_k (X_t)_k \partial_k \Lambda_t\right)\circ \Psi_t.
    \end{split}
\end{align}
Plugging the relations in \eqref{eq:manifold_almost_final} we obtain
\begin{align}
     (D\Psi_t)^{-\top}  \partial_t \Omega_t (D\Psi_t)^{-1}
     =
     \left(\partial_t\Lambda_t + \Lambda_t DX_t +DX^\top_t \Lambda_t
    + \sum_k (X_t)_k \partial_k\Lambda_t\right)\circ \Psi_t.
\end{align}
Thus we established that the relation \eqref{eq:partial_Omega}
also holds in the manifold case. The rest of the proof is the same.
}
\end{proof}
The proof of Corollary~\ref{co:undeformed} is trivial.
\begin{proof}[Proof of Corollary~\ref{co:undeformed}]
Apply Theorem~\ref{th:loc_loc} with $f_t=f_0$ constant. The assumption that $\Phi_t$
is analytic in $t$ can be dropped as explained in the proof of
Theorem~\ref{th:loc_loc}.
\end{proof}

For reference, we now state the uniqueness result for second order hyperbolic 
partial differential equations. 
Let $U\subset \R^n$ open, bounded and let $U_T=U\times (0, T)$.
Consider the boundary problem 
\begin{align}\label{eq:pde_hyperbolic}
\partial_t^2 u + Lu&=f\text{ in $U_T$}\\
\label{eq:pde_bd1}
u &= 0 \text{ in $\partial U \times [0,T]$}\\
\label{eq:pde_bd2}
u&=g, \partial_t u=h \text{ on $U\times \{0\}$}
\end{align}
where $f:U_T\to \R$ and $g,h:U\to \R$ are given functions which we assume 
to be $C^1$ and $g=0$ on $\partial U$. The function $u:U_T\to \R^d$ is the  
unknown. The operator is assumed to be an elliptic operator given by
\begin{align}
Lu = - \sum_{i,j=1}^n \partial_i (a^{ij}(x,t) \partial_j u)
\end{align}
where we assume $a^{ij}\in C^1(\bar{U}_T)$, $a^{ij}=a^{ji}$, and that 
there is $\theta >0$ such that 
\begin{align}\label{eq:elliptic}
\sum_{i,j=1}^n \xi_i\xi_j a^{ij}(x,t)\geq \theta |\xi|^2
\end{align}
for all $(x,t)\in U_T$ and $\xi \in \R^n$.
Then the following result holds.
\begin{theorem}[Theorem~4 in Section~7.2 in \cite{evans10}]\label{th:hyperbolic}
Under the assumptions above there is a unique weak solution $u$ of the system 
\eqref{eq:pde_hyperbolic} with boundary values as in \eqref{eq:pde_bd1} and \eqref{eq:pde_bd2}. 
\end{theorem}
For our purposes it is not necessary to define weak solution let us just emphasize that any classical solution is a weak solution so this implies uniqueness of classical solutions.

The key obstacle to improve upon this result and to remove the 
compact support condition on $X$ is that the resulting PDE in equation~\eqref{eq:pde2}
is well posed for the Cauchy initial value problem but it is not well posed for the Dirichlet problem
or for mixed Dirichlet and Neumann boundary data.
In particular, solutions are, in general, not unique. Furthermore, there are no general uniqueness results for first order systems 
as in \eqref{eq:diff_first_order}.
Note that the existence of a non-trivial divergence free solution $X_0$ of
\eqref{eq:diff_first_order} does not imply that a
non-constant flow $\Phi_t$ exists because this is not sufficient to 
define the flow for positive times.

We illustrate the influence of the boundary condition further below, when we prove
Theorem~\ref{th:loc_globalalt}.

\subsection{Proof of Theorem~\ref{th:loc_globalalt}}
In this section we show that a family of simple mixing functions
is locally identifiable for most parameter values even when the mixing is not known close to the boundary. 
Note that actually we can construct a set of parameter values for which this
holds giving a slightly stronger result that we state now. 
Theorem~\ref{th:loc_globalalt} will be simple consequence of this result.
\begin{theorem}\label{th:loc_global_app}
Let $f:\cube\to \R^d$ be given by $f(x)=RDx$, where $R\in \mathrm{O}(d)$ and $D=\diag(\mu_1, \ldots, \mu_d)$ with $\mu_i>0$ and $\mu_i^{-2}$ are linearly independent over the rational numbers $\mathbb{Q}$.
 Suppose that $\Phi_t$  is a smooth invariant deformation in $\mc{F}_\ima$ such that $\Phi_0=f$,
 $(\Phi_t)_\ast \nu=f_\ast \nu$, and $\Phi_t$ is analytic in $t$.
Then
 $\Phi_t=f$ on $\cube$, i.e., 
$\Phi_t$ is constant in time.
\end{theorem}
\begin{proof}
The initial part of the proof proceeds as in 
the proof of Theorem~\ref{th:loc_loc} and we keep using the same notation.
In particular $\Psi_t=(\Phi_t)^{-1}\circ f$
and $X_t$ is given by $\partial_t \Psi_t = X_t\circ \Psi_t$.

Note, that now $f_t$ is constant in $t$ and therefore $\Omega_t=Df_t^\top Df_t$
is constant in $t$.
We now investigate the boundary conditions for equation~\ref{eq:pde2}.
Note that 
\begin{align}
    \nu=(\Phi_t)_\ast^{-1} (\Phi_t)_\ast\nu=
    (\Phi_t)_\ast^{-1} f_\ast\nu
=(\Psi_t)_\ast\nu.
\end{align}
So $\Psi_t$ preserves $\nu$ and we conclude that $\Psi_t((0,1)^d)=(0,1)^d$.
Let us denote by 
\begin{align}
D_i=\{x\in [0,1]^d|\, x_i\in \{0,1\}\}
\end{align}
the boundary hyperplanes and write $D=\partial\cube=\partial (0,1)^d=\bigcup_i D_i$. As $\Psi_t$ maps $(0,1)^d$ bijectively to itself
we conclude that 
\begin{align}
(X_t)_i=0 \quad \text{on $D_i$}.
\end{align}
We now focus on $t=0$ and use the shorthand $X=X_0$.
Then the differential equation
\eqref{eq:diff_first_order_short} implies that 
\begin{align}
    \partial_i X_j=\Lambda_i/\Lambda_j\partial_j X_i=0\quad \text{ on $D_i$}.
\end{align}
We conclude that the function $X_i$ solves the following mixed Dirichlet and Neumann type boundary problem
\begin{align}\label{eq:hyperbolic_final}
    \partial_i^2 X_i - \sum_{j\neq i} \partial_j(a_j\partial_j X_i)&=0 \text{  on $(0,1)^d$}
    \\ \label{eq:boundary}
    \partial_j X_i&=0 \text{ on $D_j$ for $j\neq i$}
    \\ \label{eq:initial}
    X_i &= 0 \text{ on $D_i$.}
\end{align}
Recall here that $a_j=\Lambda_i / \Lambda_j$.
So far we have not used any specific assumption except that $\Phi_t$ is a continuous deformation and $\Phi_0\in \mc{F}_\ima$. So the existence of non-trivial continuous deformations implies
that a certain hyperbolic PDE has a non-trivial solution. Unfortunately, this type of boundary value problem for hyperbolic equations is not well posed and has not always a unique solution. 
We now show that in the specific setting of the theorem
uniqueness holds. In this case $f$ is linear and 
\begin{align}
    Df=R \;\mathrm{Diag}(\mu_1,\ldots, \mu_d)
\end{align} 
so 
\begin{align}
    \Lambda=(Df)^\top Df = \mathrm{Diag}(\mu_1^2,\ldots, \mu_d^2).
\end{align}
This implies 
\begin{align}
    a_j= \Lambda_i / \Lambda_j = \mu_i^2/\mu_j^2,
\end{align}
in particular $a_j$ is constant. So the equation \eqref{eq:hyperbolic_final}
becomes a constant coefficient hyperbolic equation which can be solved explicitly.

We can now use Theorem~1 from \cite{hyperbolic} (and a simple scaling argument) we conclude that the system
\eqref{eq:hyperbolic_final} has a unique solution which is $X_i=0$ (actually this result is for $X_i=0$ on $\partial D$ but the proof is still valid). 
To give an intuition, we note that separation of variable is possible in this setting and all solutions to the boundary value problem
\eqref{eq:hyperbolic_final} and \eqref{eq:boundary} (i.e., without the boundary condition \eqref{eq:initial} for $D_i$) 
can be expressed as a linear combination of the form
\begin{align}\label{eq:product_sol}
    X_i(s)=f(s_i)\prod_{j\neq i} \p_j(s_j)
\end{align}
where $\p_j$ are eigenfunctions of the problem $\p_j''=\lambda_j \p_j$ on $(0,1)$
and $\p_j'(0)=\p'_j(1)=0$. It is easy to see that those are
given by $\cos(\pi m_j t)$ where $m_j\in \mathbb{N}_0$
and then $\p_j''(s_j)=\pi^2 m_j^2\p_j$.
Solving for $f$ we find from \eqref{eq:hyperbolic_final} that $f$ satisfies 
the ode 
\begin{align}
    f''(s_i)=\left(\sum_{j\neq i} \pi^2\frac{\mu_i^2}{\mu_j^2} m_j^2\right) f(s_i).
\end{align}
Using now that $X_i(0)=0$ (by \eqref{eq:initial}) we conclude $f(0)=0$
and therefore
\begin{align}
    f(s_i)=C\sin(\pi \alpha s_i)
\end{align}
where $\alpha = \sqrt{\sum_{j\neq i} m_j^2 \mu_i^2/\mu_j^2}$, or equivalently
\begin{align}
    0 = \alpha^2 \mu_i^{-2}-\sum_{j\neq i} m_j^2 \mu_j^{-2}.
\end{align}
Now the condition $X_i(s)=0$ for all $s\in[0,1]^d$ with $s_i=1$
is satisfied if and only if $f(1)=0$ which holds iff $\alpha\in \mathbb{N}_0$.
Note that this argument also implies to solutions that are sums of functions as in
\eqref{eq:product_sol} by linear independence.
Then the assumption that $\mu_i^{-2}$ are linearly independent
over $\mathbb{Q}$ implies that $\alpha=0$ (and $m_j=0$) which implies $X_i=0$.
Note that this argument only applies at $t=0$ because it heavily relies on the explicit form of $\Phi_0=f$. However, we can apply the same reasoning to $\partial_t^k X_t$
inductively (just as in the proof of Theorem~\ref{th:loc_loc})
and then conclude using the assumption that $\Phi_t$ is analytic in $t$.

The complete argument goes as follows. We take the time derivative of equation
\eqref{eq:main_pde} (recall that $\partial_t\Omega_t=0$ as $f_t=f_0$ and get, denoting $\dot{X_t}=\partial_t X_t$
and $\dot{\Lambda}_t=\partial_t \Lambda_t$,
\begin{align}
    (D X_t)^\top \dot{\Lambda}_t
    +\Lambda_t DX_t
    + (D \dot{X}_t)^\top \Lambda_t
    +\Lambda_t D\dot{X}_t\in \mathrm{Diag}(d).
\end{align}
We have seen that $DX_0=0$ so we infer
\begin{align}
    (D \dot{X}_0)^\top \Lambda_0
    +\Lambda_0 D\dot{X}_0\in \mathrm{Diag}(d)
\end{align}
and $\Div \dot{X}_t=\partial_t \Div X_t=0$.
The same arguments as before imply $\dot{X}_0=0$ on $(0,1)^d$.
By induction all time derivatives of $X_0$ vanish.
This implies that $\partial_t^l\Psi_{t=0}(s)=0$ for all $s$ and $l$,
i.e., its Taylor expansion at $t=0$ disappears and since we assumed $\Phi_t$ to be analytic in $t$ so is $\Psi_t$ and we conclude that $\Psi_t(s)=\Psi_0(s)=s$
and therefore $\Phi_t(s)=f(s)$.
\end{proof}
The proof of Theorem~\ref{th:loc_globalalt} is now simple.
\begin{proof}[Proof of Theorem~\ref{th:loc_globalalt}]
Using Theorem~\ref{th:loc_global_app} we only need to show that with probability 1 the real numbers $\mu_i^{-2}$ are independent over $\Q$.
Note that by assumption $\mu_i$ has a density. Thus, also the distribution of $\alpha_i=\mu_i^{-2}$ has a density, i.e., is absolutely continuous with repsect to the Lebesgue measure.
For a vector of rational numbers $(q_1,\ldots, q_d)$
the set $H_q=\{\alpha\in\R^d\,| \sum q_i\alpha_i=0\}$ is a codimension 1 hyperplane and thus has Lebesgue measure 0. As $\Q$ is countable this implies
that $N=\bigcup_{q\in \Q^d} H_q$ is a null set. Note that $\alpha\in N$ iff
$\alpha_i$ are linearly dependent over $\mathbb{Q}$.
Since the distribution of $\alpha_i=\mu_i^{-2}$ is absolutely continuous with respect to the Lebesgue measure, we conclude that 
\begin{align}
    \P(\alpha\in N)=0
\end{align}
which implies the result.
\end{proof}
\subsection{Proofs for the construction of spurious solutions}
Finally, we show how flows can be used to construct families of solutions to the
ICA problem. This section contains the technical results missing in the overview given in Appendix~\ref{app:spurious}.

 The first construction was described in Lemma~\ref{le:const_rotations}.
Let us for completeness give a proof (we emphasize again that this result
is essentially taken from \cite{nonlinear_ica}).
\begin{proof}[Sketch of proof of Lemma~\ref{le:const_rotations}]
Note that it is sufficient to show that the maps $h_{R,a}$ are volume preserving for fixed $t$ so we ignore the time argument. 
It is easy to see that $h_{R,a}$ is bijective (the inverse is given $h_{Q,a}$
where $Q(t, r)=R(t,r)^{-1}$). 
Then we only need to show that $\Det Dh_{R,a}(s)=1$
 for all $s$. We calculate (denoting $r=|s-a|$)
\begin{align}
\begin{split}
  (Dh_{R,a}(s))_{ij}&=  \partial_j (h_{R,a})_i = R(r)_{ij} +  \sum_k (\partial_j R)_{ik}(|s-a|)(s-a)_k
  \\
  &=
  \partial_j (h_{R,a})_i = R(r)_{ij} +  \sum_k ( \partial_r R)_{ik}(r)(s-a)_k
  \partial_j |s-a|.
  \end{split}
\end{align}
We conclude  (writing $R'=\partial_r R$)
\begin{align}
    Dh_{R,a}(s)= R(r)+ R'(r)(s-a)\otimes \nabla |s-a|
    =R(r)+\frac{1}{|s-a|} R'(r)(s-a)\otimes(s-a).
\end{align}
Then we obtain, using the matrix determinant lemma for rank 1 updates
($\Det (A+u\otimes v)=(1+u\cdot A^{-1}v)\Det A$
\begin{align}
    \Det Dh_{R,a}(s)=\left(1+ \frac{1}{|s-a|}(s-a) R(r)^\top R'(r) (s-a)\right)\Det(R(r)).
\end{align}
Now we use that $R(r)\in \mathrm{O}(d)$ so $\Det(R(r))=1$
and $R(r)^{-1}=R(r)^\top$. Differentiating $R(r)^\top R(r)=\mathrm{Id}_d$
with respect to $r$ we conclude that $R(r)^\top R'(r)$ is skew which implies
\begin{align}
    (s-a) R(r)^\top R'(r) (s-a)=0.
\end{align}
We have therefore shown $\Det Dh_{R,a}(s)=1$, completing the proof.
\end{proof}

Now we give another construction that also establishes Fact~\ref{fa:non_uniqueness} 
based on suitable divergence free vector fields.
All we need to construct is divergence free vector fields with compact support. 
Consider any smooth function $\p:\R^d\to \R$ such that its support is contained in 
$\Omega$.  Then we consider the vector fields $X^{ij}:\R^d\to \R^d$
for $1\leq i<j\leq d$ given by
\begin{align}
X^{ij}_i = \partial_j \p, \quad X^{ij}_j=-\partial_j \p, \quad X^{ij}_k=0 \quad \text{for $k\notin\{i,j\}$}.
\end{align}
Then we get $\Div X^{ij}=\partial_i\partial_j\p-\partial_j\partial_i\p=0$.
So those vector fields are divergence free and we conclude that the space
\begin{align}
\mc{X}=\{X:\R^d\to\R^d| \mathrm{supp}(X)\subset \Omega, \;\Div X=0\}
\end{align}
is infinite dimensional. Every $X\in \mc{X}$ generates a flow
$\Phi_t$ defined by
\begin{align}
\partial_t \Phi_t = X(\Phi_t), \quad \Phi_0(s)=s.
\end{align}
Using equation \eqref{eq:continuity}
we conclude that $(\Phi_t)_\ast \nu=\nu$
because $\nu$ has a constant density and the support condition of $X$ ensures
that $\Phi_t((0,1)^d)=(0,1)^d$.
Then the family $f_t=f\circ \Phi_t$ has the property
that $(f_t)_\ast \nu=f_\ast \nu$.
Note that this construction can be easily generalised to 
source distributions $\P$ with differentiable density $p$. 
In this case the condition $\Phi_t\P=\P$ is satisfied when
$\Div(pX)=0$. Clearly it is sufficient to consider $X= Y/p$
where $Y\in \mc{X}$ (assuming that $p>c$ for some $c>0$ on $\Omega$).

\section{Proofs for the result on volume preserving maps}\label{app:volume}

\begin{figure}
    \centering
    \includegraphics[width=.75\textwidth]{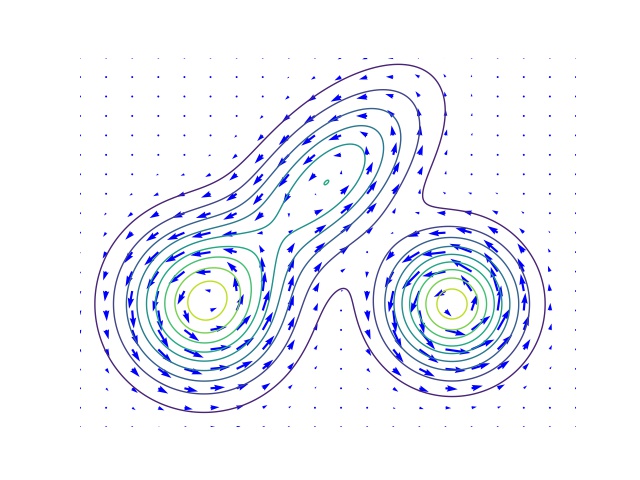}
    \caption{A sketch of the vector fields $X^{ij}$ for $d=2$ constructed in 
    the proof of Theorem~\ref{th:volume}. The closed lines are level lines of the probability density (which is a Gaussian mixture here). Note that the vector field is parallel to the level lines and its magnitude proportional to the norm of the gradient of the density.}
    \label{fig:final}
\end{figure}
Next we show that this construction can be generalised to 
volume preserving transformations and we prove Theorem~\ref{th:volume}.
Note that in the special case that the distribution of $s$ is $\nu$ 
the construction above already works. This is a special case because 
the condition $(f_t)_\ast \nu = f_\ast\nu$ already implies that $f_t$ is 
volume preserving as soon as  $f$ is volume preserving as the density of $\nu$ is constant.
So in this case the condition that $f_t$ is volume preserving and $(f_t)_\ast\nu=f_\ast\nu$ essentially agree which is not the case for general base measures.
\begin{proof}[Proof of Theorem~\ref{th:volume}]
We define a suitable vector field explicitly. 
Consider $X^{ij}:\R^d\to\R^d$ for $1\leq i<j\leq d$ defined by
\begin{align}
X_k^{ij}=
\begin{cases}
\partial_j p \quad k=i
\\
-\partial_i p \quad k=j
\\
0\quad k\notin\{i,j\}.
\end{cases}
\end{align}
An illustration of this vector field is given in Figure~\ref{fig:final}.
We consider the family of functions $f_t= f\circ \Phi_t^{ij}$ where
the flow $\Phi_t^{ij}$ is defined by
$\Phi_0^{ij}(s)=s$ and $\partial_t \Phi_t^{ij}(s)=X(\Phi_t^{ij}(s))$.
Note that boundedness of $\nabla p$ and $p\in C^2$ imply that $\Phi_t$ exists globally 
and defines a diffeomorphism.
We claim that $\Phi_t^{ij}$ satisfies 
$\Det \Phi_t^{ij}(s)=1$ for all $s$ and $(\Phi_t^{ij})_\ast \P=\P$.
The former condition means that $\Phi^{ij}$ preserves the standard volume (Lebesgue-measure) which is the case if $\Div(X^{ij})=0$ 
while the second relation is satisfied if $\Div(p X^{ij})=0$ by equation \eqref{eq:continuity}. 
We calculate
\begin{align}
\Div X^{ij}=\partial_i\partial_j p -\partial_j\partial_i p = 0.
\end{align}
We also find
\begin{align}
\Div( X^{ij} p ) = p \Div(X^{ij}) + X^{ij}\cdot \nabla p
= \partial_j p \partial_i p - \partial_i p\partial_j p=0.
\end{align}
This ends the proof. 
\end{proof}
To give an example, we consider $d=2$ and $\P$ with rotation invariant density
$p(s)=p(|s|)$. Then $X(s)= f(|s|)s^\perp$ where $s^\perp =(s_2,-s_1)^\top$
and the flow lines are circles around the origin where the speed depends on the radius through the derivative of $p(|s|)$.
Let us add some remarks concerning this result.
\begin{remark}
\begin{enumerate}
\item The constructed flows are non-trivial, i.e., not constant because the probability density cannot be constant (as we assumed it to be $C^2$) and thus $X^{ij}$
is not identically vanishing.
\item It is easy to see (e.g., through the example above) that the flows $\Phi^{ij}$ will, in general, mix the coordinates $i$ and $j$ thus this really shows that ICA is not identifiable for volume preserving maps.
\item While we construct a finite family of solutions they can be combined,
e.g., 
\begin{align}
f' = f\circ \Phi^{i_1j_1}_{t_1}\circ\ldots\cdot \Phi^{i_kj_k}_{t_k}
\end{align}
to yield a large space of solutions.
\item By choosing coordinates cleverly, it is possible to construct a
vector field $X$ satisfying $\Div (X)=\Div(pX)=0$ with compact support.
So even knowing $f$ close to the boundary of the support of $\P$ is not sufficient  
to uniquely identify $f$. 
\item While it is not possible to identify ICA using volume preserving transformations, it can be possible to identify $f(s)$ for certain values of $f$ if $\P$ is known. If $p$ has a unique maximum
at $s_0$ then $x_0=f(s_0)$ will be the point with the largest density of $x$ because
volume preserving transformations transform the density trivially (see \eqref{eq:density_pushforward}). 
\end{enumerate}
\end{remark}

Let us finally sketch a proof of Proposition~\ref{prop:non_rigid}.
\begin{proof}[Proof of Proposition~\ref{prop:non_rigid}]
We assume in addition that the line segment $t_i=\{x_i+\lambda(y_i-x_i)|\, \lambda\in [0,1]\}$ does not contain any $x_j$ or $y_j$ for $j\neq i$.
The generalisation to the general case is straightforward, e.g., by composing
two diffeomorphisms as constructed here.
It is clearly sufficient to construct volume preserving diffeomorphisms
$h_i$ such that $h_i(x_i)=y_i$ and $h_i(x_j)=x_j$, $h_i(y_j)=y_j$ for
$j\neq i$ which can then be composed.
We consider the vector field $X_i=(y_i-x_i)\p$ where $\p$ is a smooth cut-off
function with $\p(x)=1$ for $x\in t_i$ and
$\p(x_j)=\p(y_j)=0$ for $j\neq i$. Using Theorem~2 in \cite{moser} we conclude
that there is $Y_i$ such that $\mathrm{supp}(Y_i)\subset \mathrm{supp}(\nabla\p)$ and $\Div(X_i+Y_i)=0$. Considering the flow of $Z_i=X_i+Y_i$ 
up to time $1$ we obtain a function $h_i$ as desired.
Indeed, $Z_i(x_j)=Z_i(y_j)=0$ for $j\neq i$ because $x_j$ and $y_j$ are by construction of $\p$ outside the support of $X_i$ so $X_i(x_j)=X_i(y_j)=Y_i(x_j)=Y_i(y_j)=0$
for $i\neq j$.
Moreover, for $x\in t_i$ we get $Z_i(x)=X_i(x)+Y_i(x)=y_i-x_i$.
This implies that the flow $\Phi_t^i$ of $Z$ satisfies $\Phi_t^i(x_i)=x_i+t(y_i-x_i)$
for $t\in [0,1]$.
In particular, $h_i=\Phi_t^i$ is as desired.
\end{proof}

{
\section{Experimental illustration of local identifiability}\label{app:ill}

\begin{figure}[h]
\centering
\includegraphics[width=\textwidth]{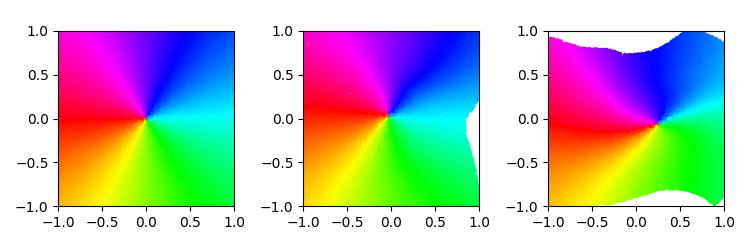}
\includegraphics[width=\textwidth]{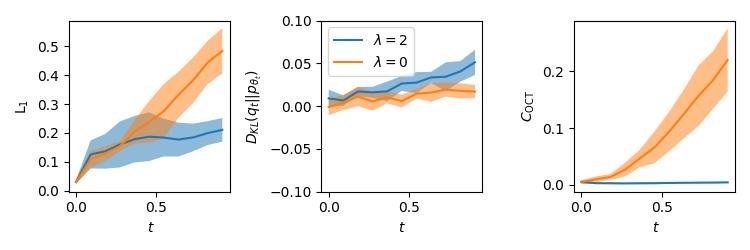}
\caption{Experiment with $f_t^{\mathrm{pol}}$. \textbf{Top row: (left)} ground truth of 
latent variables \textbf{(middle)} reconstructed sources $g_\theta^{-1}\circ f_{t=1}^{\mathrm{pol}}$ for $C_{\mathrm{OCT}}$
regularized training \textbf{(right)}
reconstructed sources for unregularized training
\textbf{Bottom row:} Orange curves: Unregularized training. Blue curves: $C_{\ima}$-regularized training. \textbf{(left)}
$L_1$ distance ground truth - reconstruction over time (see \eqref{eq:L1})
\textbf{(middle)} Forward KL-divergence over time
\textbf{(right)} $C_\mathrm{OCT}$ (see \eqref{eq:cima} over time.
}
\label{fig:pol}
\end{figure}

\begin{figure}[h]
\centering
\includegraphics[width=\textwidth]{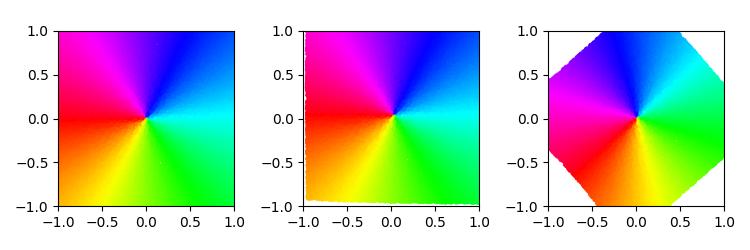}
\includegraphics[width=\textwidth]{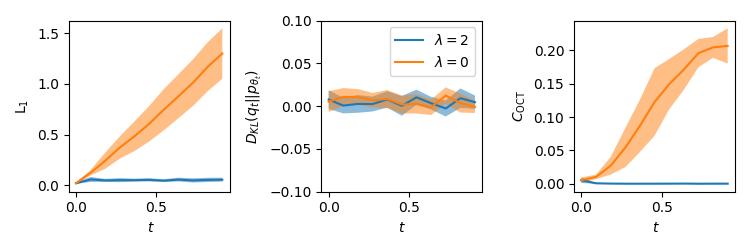}
\caption{Same as Figure \ref{fig:pol}
for $f_t^{\mathrm{rot}}$. \textbf{Top row: (left)} ground truth of 
latent variables \textbf{(middle)} reconstructed sources $g_\theta^{-1}\circ f_{t=1}^{\mathrm{pol}}$ for $C_{\mathrm{OCT}}$
regularized training \textbf{(right)}
reconstructed sources for unregularized training
\textbf{Bottom row:} Orange curves: Unregularized training. Blue curves: $C_{\ima}$-regularized training \textbf{(left)}
$L_1$ distance ground truth - reconstruction over time (see \eqref{eq:L1})
\textbf{(middle)} Forward KL-divergence over time
\textbf{(right)} $C_\mathrm{OCT}$ (see \eqref{eq:cima} over time.
}
\label{fig:rot}
\end{figure}
We provide a toy experiment to illustrate the meaning and significance of Theorem~\ref{th:loc_loc}. Note that this shall just underpin this specific setting, 
for general experiments concerning the usefulness of orthogonal coordinate transforms we refer to \cite{ima}.
We consider a function $f_0\in \mc{F}_{\ima}$ and then assume that 
there is a smooth time-dependent transformation $f_t$ such that  $f_t\in \mc{F}_{\ima}$
($f_t$ is a smooth invariant deformation in the language of the paper). We only observe the changing output distribution but the latent sources are unobserved.

Suppose, however, that we know the initial mixing $f_0$, i.e.,  we trained an initial model $\Phi_0$, s.t., $\Phi_0=f_0$. 
Then we train $\Phi_t$ starting from $\Phi_0$ such that
$\Phi_t(s)\indistribution f_t(s)$ where $s$ is distributed according to some base distribution $\mu$ of the latent variables. If we in addition enforce 
$\Phi_t\in \mc{F}_{\ima}$ Theorem~\ref{th:loc_loc} essentially tells us that 
$\Phi_t=f_t$ while no such guarantee exists without the functional restriction. 
Here, we verify experimentally that this is indeed the case in a simple setting.
\paragraph{Detailed Experimental Setting}
We work in dimension 2 and consider a standard normal  base distribution $\mu$.
We consider polar coordinates for $f$
\begin{align}
    f_{\mathrm{polar}}(r, \p) &= (r\sin(\p), r\cos(\p)).
\end{align}
We then define $f_t^{\mathrm{pol}}=f_{\mathrm{polar}} \circ h_t$ where $h_t$ is a coordinate-wise 
transformation defined for $0\leq t\leq 1$ by
\begin{align}
    h_t(s_1, s_2) = \left(s_1+\frac{t}{2}\sin(s_1+t) + 3 , \frac{s_2+t}{2}\right).
\end{align}
Note that we shift the first coordinate and scale the angular coordinates 
to ensure that the map is injective (except for the light tail of the Gaussian).
We also consider the setting where 
\begin{align}
    f_t^{\mathrm{rot}}(s_1, s_2) = e^{t W} (2s_1,s_2)
\end{align}
and 
\begin{align}
    W= \begin{pmatrix} 0&1\\ -1&0
    \end{pmatrix}.
\end{align}
To model $\Phi_t$ we use a normalizing flow model \cite{original_flows, review_flows}. while this is also convenient for the experiments, it is important
to not use models that implicitly promote
orthogonal columns of the Jacobian, e.g.,
VAEs (see end of Section~\ref{sec:setting})
to extract the effect
of enforcing $\Phi_t\in \mc{F}_{\ima}$.
We write
$\Phi_t(s)=g(\theta_t, s)=g_{\theta_t}(s)$ where $\theta$ denote the trainable parameters of the flow which will vary with time.
For our implementation we use nflows \cite{nflows} and we use 5 masked affine autoregressive transformation layers 
with 15 hidden features followed by random permutations.
Then the following procedure is used.
We train the normalizing flow such that  $f_0=g(\theta_0,\cdot)$.
We assume the base distribution of the flow is $\mu$ 
and denote the induced distribution $(g_\theta)_\ast \mu$ by $p_\theta$.
We discretize the time interval in 10 intervals with endpoints $t_1$ to $t_{10}$.
Iteratively we train $g(\theta_{t_i},\cdot)
=g(\theta_i,\cdot)$ starting from $g(\theta_{i-1},\cdot)$
to maximize the likelihood of observations from $x_{i}\sim(f_{t_{i}})_\ast \mu$, i.e., we 
consider the loss
\begin{align}
    L_{\mathrm{ML}}(\theta) = \mathbb{E}_{x_{i}} (-\log p_\theta(x_{i})).
\end{align}
We do this without regularisation and with a regularisation that promotes
$g_{\theta_t}\in \mc{F}_{\ima}$. Here we use the IMA contrast introduced in \cite{ima}
which we will call $C_{\ima}$ for conciseness. It is defined by
\begin{align}\label{eq:cima}
    C_{\ima}(f, \mu) = \int \mu(s) \left(\sum_k \log |\partial_k f| - \log \Det Df\right)
\end{align}
and we consider the total loss
\begin{align}
    L_\mathrm{Reg} = L_{\mathrm{ML}}+ \lambda \cdot C_\ima (g_\theta(\cdot), \mu).
\end{align}
We will use $\lambda=0$ and $\lambda=2$
to train a regularized and an unregularized model.
Note that $C_{\ima}$ is non-negative and vanishes exactly on OCTs (see Prop. 4.4 in \cite{ima}).

\paragraph{Results}
To measure how well the model recovers the ground truth sources we
consider 
\begin{align}\label{eq:L1}
    L_1(\theta, t) = \int \mu(s) |g^{-1}_\theta(f_t(s)) - s|.
\end{align}
Note that more complicated measures like the Amari distance are not necessary because our initialization removed the permutation symmetries.
We also consider $C_\ima(g_\theta)$ to approximate the distance to $\mc{F}_\ima$ and
the forward KL-divergence 
$D_{KL}(q_i|| p(\theta_i,\cdot))=
\mathbb{E}_{q_i(x)}(\log(q_i(x))-\log(p(\theta_i,x)) $
where $q_i=(f_i)_\ast\mu$ denotes the observational distribution at time $t_i$.
Figures \ref{fig:pol} and \ref{fig:rot}
indicate the results for the maps $f_t^{\mathrm{pol}}$ and $f_t^{\mathrm{rot}}$, respectively. Note that in both cases, the regularized and the unregularized model have small forward KL-divergence to the observational distribution, i.e., they both track the changing observational distribution. However, the regularized model recovers the ground truth latent variables more faithfully, while the unregularized model evolves towards a spurious solution. Finally, we note that the regularizer indeed ensures orthogonality of the columns of the Jacobian, i.e., $g_\theta\in \mc{F}_{\ima}$ (small
$C_{\ima}$).

\paragraph{Training details}
For training  we use stochastic gradient descend with the ADAM-optimizer \cite{adam} and train for 1000 steps with a batch size of 256 where we generate i.i.d.\ samples from the observational distribution in each step. We averaged our results over 10 runs. Total compute time was less than 24h on a workstation.

}
\end{appendix}
\end{document}